\documentclass[11pt]{article} 
\usepackage{geometry} 
\usepackage{amsmath,amsthm}
\usepackage{setspace}               
\usepackage{graphicx}
\usepackage{amssymb}
\usepackage{epstopdf}
\usepackage{wrapfig}
\usepackage{float}
\usepackage{url}
\newtheorem{thm}{Theorem}
\def\ran{\rangle}
\def\lan{\langle}

\def\si{\sigma}
\def\eps{\epsilon}
\def\wa{W_1}
\def\wb{W_2}
\def\wat{W_1^t}
\def\wbt{W_2^t}
\def\rb{R_2}
\def\lri{\lambda_{R,i}}

\def\La{\Lambda_1}
\def\Lb{\Lambda_2}
\def\Lt{\Lambda_T}
\def\Lr{\Lambda_R}
\def\la{\lambda_1}
\def\lb{\lambda_2}
\def\lt{\lambda_T}
\def\lr{\lambda_R}

\def\ll{{\mathcal L}}
\def\nn{{\mathcal N}}
\def\ww{{\mathcal W}}

\newcommand{\R}{\mathbb{R}}
\title{Deep learning with asymmetric connections and Hebbian updates}
\author{Yali Amit \\ University of Chicago}

\graphicspath{{"./Figures/"}}
\begin{document}
\maketitle

\begin{abstract}
We show that deep networks can be trained using  Hebbian updates yielding similar performance to ordinary back-propagation on challenging image datasets. To overcome the unrealistic symmetry in connections between layers, implicit in back-propagation, the feedback weights are separate from the feedforward weights. The feedback weights  are also updated with a local rule, the same as the feedforward weights -  a weight is updated solely based on the product of activity of the units it connects.  With fixed feedback weights as proposed in \cite{Lilli} performance degrades quickly as the depth of the network increases.  If the feedforward and feedback weights are initialized with the same values, as proposed in \cite{zipser1990}, they remain
the same throughout training thus precisely implementing back-propagation. We show that  even when the weights are initialized differently and at random,
and the algorithm is no longer performing back-propagation,  performance is comparable on challenging datasets.
We also propose a cost function whose derivative can be represented as a local Hebbian update on the last layer. Convolutional layers are updated with tied weights across space, which is not biologically plausible. We show that similar performance is achieved with untied layers, also known as locally connected layers, corresponding to the connectivity implied by the convolutional layers, but where weights are untied and updated separately. In the linear case we show theoretically that the convergence of the error to zero is accelerated by the update of the feedback weights.
\end{abstract}

\section{Introduction}

The success of multi-layer neural networks (deep networks) in a range of prediction tasks as well some observed
similarities observed between the properties of the network units and cortical units (\cite{DiCarlo}),  has raised the question of whether they 
can serve as models for processing in the cortex \cite{Kriegeskorte2015,Kording}.
The feedforward architecture of these networks is clearly consistent with models of neural computation: a hierarchy of layers, where the units in each layer compute their activity in terms of the weighted sum of the units of the previous layer. The main challenge with respect to biological plausibility is in the way these networks are trained.

Training of feedforward networks is based on a loss function that  compares the output of the top layer of the network to a target.
Small random subsets of training data are then used to compute the gradient of the loss with respect to the weights of the network, and these are then updated by moving a small distance
in  the opposite direction of the gradient. Due to the particular structure of the function represented by these multi-layer networks the gradient is computed using back-propagation - an algorithmic formulation of the chain rule for differentiation \cite{hinton-rumel}. In the feedforward step the
 input is passed {\it bottom-up} through the layers of the network to produce the output of the top layer and the loss is computed. Back-propagation  proceeds top-down
  through the network. Successively two things occur in each layer: first, the unit activity in the layer is updated  in terms of the layer above - feedback, then  the weights feeding into this layer are updated. The gradient  of each weight is a product of the activity of the units it connects - the feedforward pre-synaptic activity of the input unit in the lower layer  and the feedback activity in the post-synaptic unit in the current layer. In that sense the gradient computation has the form of local Hebbian  learning. However, a fundamental element of back-propagation is not biologically plausible as explained in \cite{zipser1990,Lilli}. The feedback activity of a unit is computed as a function of the units in the layer above it in the hierarchy in terms of the {\it same} weight matrix used to compute the feedforward signal, implying a symmetric synaptic connectivity matrix. 
  
  Symmetry of weight connection is an unrealistic assumption. Although reciprocal physical connections between neurons are more common than would be expected at random, these connections are physically separated in entirely different regions of the neuron and can in no way be the same. The solution proposed both in \cite{zipser1990} and in \cite{Lilli} is to create a separate system of feedback connections. The latter model is simpler in that the feedback connections are not updated so that the top-down feedback is always computed with the same weights. The earlier model proposes to update the feedback weights with the same
  increment as the feedforward weights, which  as mentioned above has a Hebbian form. Assuming they are initialized with the same values, they will always have the same value. This guarantees that the back-propagation computation is executed by the network, 
  but in effect reintroduces exact weight symmetry in the back-door, and is unrealistic. In contrast, the computation in \cite{Lilli} does not replicate
  back-propagation, as the feedback weights never change, but the price paid is that in deeper networks it performs quite poorly.
  
  The main contribution of  this paper is to experiment with the idea proposed in \cite{zipser1990}, but initialize the feedforward and feedback weights randomly (thus differently). We call this updated random feedback (URFB).
  We show that  even though the feedback weights are never replicates of the feedforward weights, the network performance is comparable to back-propagation, even with deep networks on challenging
  benchmark dataset such as CIFAR10 and CIFAR100 \cite{cifar}. In contrast, the performance with fixed weights -fixed random feedback (FRFB), as in \cite{Lilli}, degrades with depth.  
   It was noted in \cite{Lilli} that in shallow networks the feedforward weights gradually align with the fixed feedback weights so that in the long
 run an approximate back-propagation is being computed, hence the name {\it feedback alignment}. We show in a number of experiments
 that this alignment phenomenon is much stronger in URFB even for deep networks. However we also
 show that from the very initial iterations of the algorithm, long before the weights have aligned, the evolution of both the training and validation errors is comparable to that of back-propagation. 

In our experiments we  replace the commonly used unbounded rectified linear unit, with a saturated linearity $\si(x) = \min(\max(x,-1),1)$, which is more biologically plausible, as it is not unbounded,
we avoid normalization layers whose gradient is quite complex and not easily amenable to neural computation,
and we run all experiments with the simplest stochastic gradient descent that does not require any memory of earlier gradients. 
 We also experiment with randomly zeroing out half of the connections, separately for feedforward and feedback connections. Thus not only are the
feedforward and feedback weights different, but connectivity is asymmetric. 
In a simplified setting we  provide a mathematical argument for why the error decreases faster with updated feedback weights compared to fixed feedback weights.
  
  Another issue arising in considering the biological plausibility of multilayer networks is how the teaching signal is incorporated in learning.
  The primary loss used for classification problems in the neural network literature is the cross-entropy of the target with respect to the {\it softmax} of the output layer (see section \ref{loss}). The first step in back-propagation is computing
  the derivative of this loss with respect to the activities of the top layer. This derivative, which constitutes the feedback signal to the top layer, involves the computation of the softmax - a ratio of sums of exponentials of the activities of all the output units. It is not a local
  computation and is difficult to model with a neural network. 
  As a second contribution we experiment with an alternative loss, motivated by the original perceptron loss, where the feedback signal is computed locally
  only in terms of the activity of the top-level unit and the correct target signal. It is based on the one-versus all method commonly used with
  support vector machines in the mult-class setting and has been implemented through network models in \cite{amit-vr, fusi-senn, amit-walker}.
   
 Finally, although convolutional layers are consistent with the structure of retinotopic layers in visual cortex, back-propagation through these layers is not biologically plausible. Since the weights of the filters applied across space are assumed identical, the gradient of the unique filter is computed as the sum of the gradients at each location. In the brain the connections corresponding to different spatial locations are physically different and one can't expect them to undergo coordinated updates, see \cite{Lillicrap3}. 
This leads us to the final set of experiments where instead of purely convolutional layers
we use a connectivity matrix that has the sparsity structure inherited from the convolution but the values in the matrix are `untied' undergo independent local updates. Such layers are also called {\it locally connected} layers and have been used in \cite{Lillicrap3} in experiments with biologically plausible architectures. The memory requirements of such layers are much greater
than for convolutional layers, as is the computation, so for these experiments we restrict to simpler architectures. Overall
we observe the same phenomena as with convolutional layers, namely the update of the feedback connections yields performance close to that of regular back-propagation.

 The paper is organized as follows.
In section \ref{rel} we describe related work.
 In section \ref{URFB} we describe the structure of a feedforward network, the back-propagation training
 algorithm and explain how it is modified with separate
 feedback weights. We describe the loss function and explain why it requires only local
 Hebbian type updates.
  In section \ref{exp} we report a number of experiments and illustrate some interesting properties of these networks. We show that performance of URFB is lower but close to back-propagation even in very deep networks, on more challenging data sets that actually require a deep network to achieve good results.
 We show that using locally-connected layers works,  although not as well as convolutional networks,
 and that the resulting filters although not tied apriori show significant similarity across space.
 We illustrate the phenomenon of weight alignment that is much more pronounced in URFB.
 In section \ref{theory} we describe a simplified mathematical framework to study the properties
 of these algorithms and show some
 simulation results that verify 
 that updating the feedback connections yields faster convergence  than  fixed feedback connections. 
 We conclude with a discussion.
 Mathematical results and proofs are provided in the Appendix.
 
\section{Related work}\label{rel}

As indicated in the introduction, the issue of the weight symmetry required for  feedback computation in back-propagation, was already
raised by \cite{zipser1990} and the idea of separating the feedback connections from the feedforward connections was proposed.
They then suggested updating each feedforward connection and feedback connection with the same increment. Assuming all weights are initialized
at the same value the resulting computation is equivalent to back-propagation. The problem is that this reintroduces the implausible symmetry since
the feedback and feedforward weights end up being identical.

In \cite{Lilli} the simple idea of having fixed random feedback connections was explored and found to work well for shallow networks.
However, the performance degrades as the depth of the network increases.
 It was noted that in shallow networks the feedforward weights gradually align with the fixed feedback weights so that in the long
 run an approximate back-propagation is being computed, hence the name {\it feedback alignment}.
In \cite{poggio-BP}  the performance degradation of feedback alignment with depth was addressed by using layer-wise normalization of the outputs. This
yielded results with fixed random feedback FRFB that are close to momentum based gradient descent
of the back-propagation algorithm for certain network architectures. However the propagation of the gradient through the normalization layer is complex
and it is unclear how to implement it in a network.
Furthermore \cite{poggio-BP} showed that a simple transfer of information on the sign of the actual back-propagation gradient yields an improvement on 
using the purely random back-propagation matrix. It is however unclear how such information could be transmitted between different synapses.

In \cite{bogacz} a model for training a multilayer network is proposed using a predictive coding framework.
However it appears that the model assumes symmetric connections, i.e. the strength of the connection
from an error node and a variable in the preceding layer is the same as the reverse connection. 
A similar issue arises  in \cite{fbu}, where in the analysis of their algorithm, they assume
that in the long run, since the updates are the same, the synaptic values are the same. This is approximately true, in the sense that the correlations between feedforward and feedback weights increase but
significant improvement in error rates are observed even early on when the correlations are weak.

\cite{burbank} implements a proposal similar to \cite{zipser1990} in the context of an autoencoder and attempts to find STDP rules that can implement
the same increment for the feedforeward and feedback connections. Again it is assumed that the initial conditions are very similar so that
at each step the feedforward and feedback weights are closely aligned.

In a recently archived paper \cite{Pozzi} also goes back to the proposal in \cite{zipser1990}.
However, as in our paper,  they experiment {with \it different} initializations of the feedforward and feedback connections.
They introduce a pairing of feedback and feedforward units to model the gating of information from the feedforward
pass and the feedback pass. Algorithmically, the only substantial difference to our proposal is in the error signal produced by the output layer, only connections to the output unit representing the correct class are updated. 

Here we show that there is a natural
way to update all units in the output layer so that subsequent synaptic modifications in the back-propagation are all Hebbian.
The correct class unit is activated at the value 1 if the input is below a threshold, and the other classes are activated
as $-\mu$ if the input is above a threshold. Thus, corrections occur through top-down feedback in the system when the inputs of
any of the output units are not of sufficient magnitude and of the correct sign.
We show that this approach works well even in much deeper networks with several convolutional layers and with more challenging
data sets. We also present a mathematical analysis of the linearized version of this algorithm and show that
the error converges faster when the feedback weights are updated compared to when they are held fixed as in \cite{Lilli}.

\cite{Lee,Lillicrap3} study target propagation where
an error signal is computed in each hidden unit as the difference between the feedforward activity of 
that unit and a target value propagated from  above with feedback connections that are separate from the feedforward
connections. The feedback connections between each two consecutive layers are trained to approximate the inverse
of the feedforward function between those layers, i.e. the non-linearity applied to the linear transformation of the lower layer.
In \cite{Lillicrap3} they analyze the performance of this
method on a number of image classification problems and use locally connected layers instead of convolutional layers.
In target propagation the losses for both the forward and the backward connections rely on magnitudes of differences between signals
requiring a more complex synaptic modification mechanism than simple products of activities of pre and post synaptic neurons as proposed
in our model.

Such synaptic modification mechanisms are studied in \cite{Lillicrap2}. A biological model for the neuronal units is presented that combines the feedforward
and feedback signals within each neuron, and produces of an error signal assuming fixed feedback weights as in \cite{Lilli}.
The idea is to divide the neuron into two separate compartments one computing feedforward signals and one computing
feedback signals, with different phases of learning involving different combinations of these two signals.
In addition to computing an error signal internally to the neuron this model  avoids the need to compute signed errors, which imply negative as well as positive neuronal
activity. However this is done by assuming the neuron can internally compute the difference in average voltage between two time
intervals. In \cite{Sacr-Senn} this model is extended to include an inhibitory neuron attached to each hidden unit neuron with plastic synaptic connections to and from the hidden unit. They claim that this eliminates the need to compute the feedback error in separate phases form the feedforward error. 

In our model we simply assume that once the feedforward phase is complete the feedback signal {\it replaces} the feedforward signal
at a unit - at the proper timing - to allow for the proper update of the incoming feedforward and outgoing feedback synapses.

\section{The updated random feedback algorithm}\label{URFB}

In this section we first describe the structure of a multilayer network, how the back-propagation algorithm works
and how we modify it to avoid symmetric connections and maintain simple Hebbian updates to both feedforward and feedback connections.
We then describe a loss function, whose derivatives can
be computed locally, yielding a Hebbian input dependent update of the
weights connecting to the final output layer. 
 
 \begin{figure}[h]
\center
\includegraphics[height=3in]{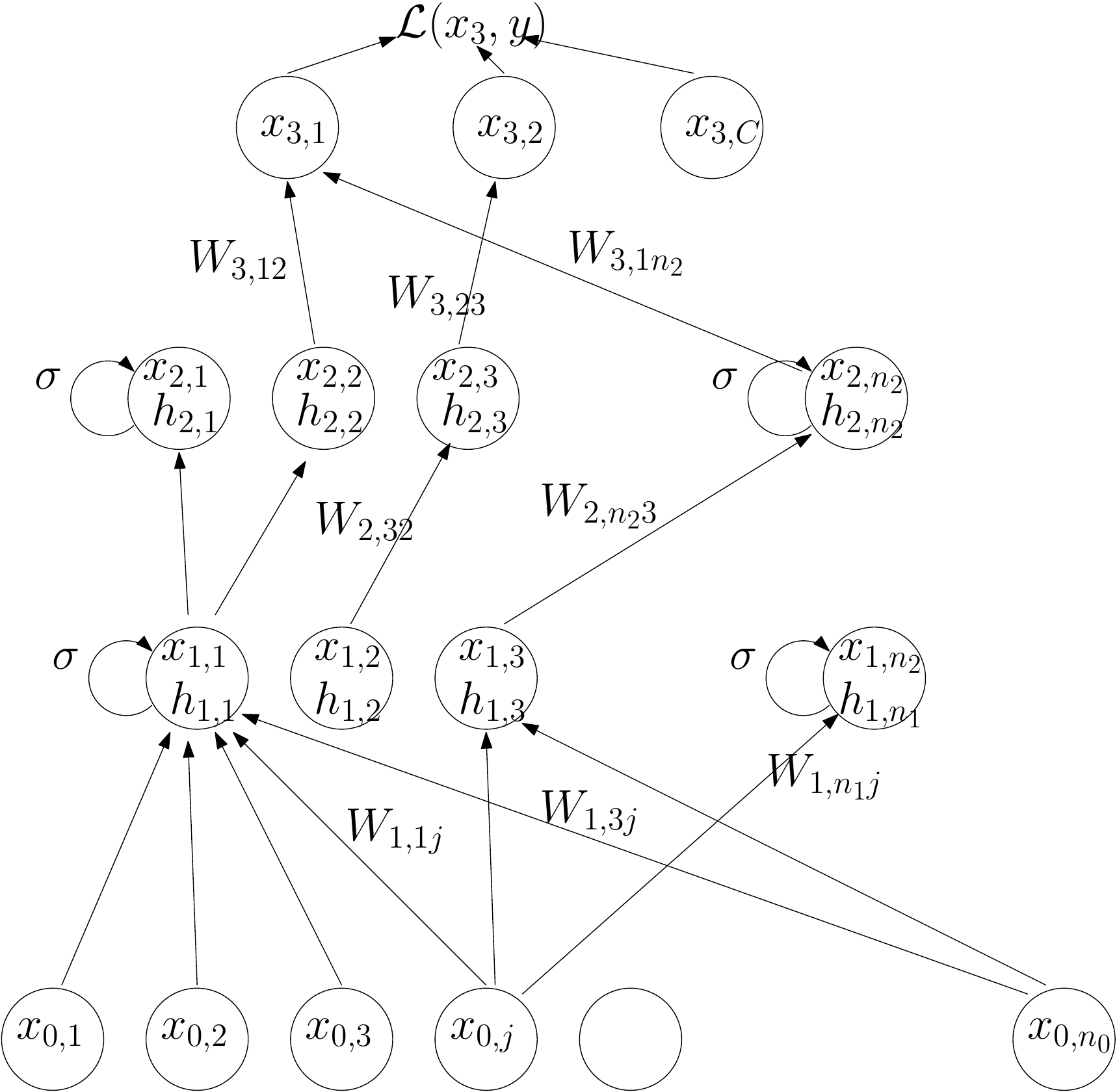}
\caption{An illustration of the computations in a feedforward network}\label{feedfor}.
\end{figure}

\subsection{Updated asymmetric feedback connections}\label{as}

A multi-layer network is composed of a sequence of layers $0,\ldots,L$. The data at the input layer is denoted $x_0$.
Each layer is composed of $n_l$ units.
Let $W_{l,ij}$ be the feedforward weight connecting unit $j$ in layer $l-1$ to unit $i$ in layer $l$.
Let $x_l, l=1,\ldots,L$ be the output of layer $l$, this is computed as
\begin{align}\label{ff}
x_{l,i} = \sigma(h_{l,i}), \quad h_{l,i} = & \sum_{j=1}^{n_{l-1}} W_{l,ij} x_{l-1,j}, i=1,\ldots,n_l. \nonumber \\
\text{or} \quad \quad \quad  h_l = &W_l x_{l-1},
\end{align}
where $\sigma$ is some form of non-linearity and $W_l$ is the $n_l \times n_{l-1}$ matrix of weights connecting layer $l-1$ to layer $l$. We denote $h_{l,i}$ the input of unit $i$ of layer $l$.
For classification problems with $C$ classes the top layer $L$, also called the output layer,
 has $C$ units $x_{L,1},\ldots,x_{L,C}$. In this last layer no non-linearity is applied, i.e. $x_{L,i} = h_{L,i}.$
For given input $x_0$ we can write $x_L = \nn(x_0,\ww)$, where $\nn$ represents the 
function computed through the multiple layers of the network with the set of weights $\ww$.
The classifier is then defined as:
$$  \hat c(x_0) = \text{argmax}_{i} x_{L,i} = \text{argmax}_i \nn(x_0,\ww).$$
A feedforward network with 3 layers is shown in figure \ref{feedfor}

We define
a loss $\ll(x_L,y,\ww)$ comparing the activity of the output layer to a target value,
an indicator vector denoting the correct class of the input.
At each presentation of a training example the derivative $\partial \ll/ \partial W_{l,ij}$ of the loss with respect to each
weight is computed, and the value of the
weight is updated as
$$W_{l,ij} = W_{l,ij} - \eta \partial \ll/ \partial W_{l,ij},$$
where $\eta$ is a small scalar called the time-step or learning rate.
This is done in two phases. In the first phase, the feedforward phase, the input $x_0$ is presented at layer $l=0$ and passed
successively through the layers $l=1,\ldots,L$ as described in \eqref{ff}. 
In the second phase the derivatives are computed starting with $W_{L,ij}$ for the top layer and
successively moving down the hierarchy. At each layer the following two equalities hold due to the chain rule for differentiation:
\begin{align*}
\frac{\partial \ll}{\partial W_{l,ij}} = & \frac{\partial \ll}{\partial h_{l,i}} \frac{\partial h_{l,i}}{\partial W_{l,ij}} = 
\frac{\partial \ll}{\partial  h_{l,i}} x_{l-1,j} \nonumber \\
\frac{\partial \ll}{\partial h_{l,i}} = &\sigma'(h_{l,i}) \sum_{k=1}^{n_{l+1}} \frac{\partial \ll}{\partial h_{l+1,k}} W_{l+1,ki}.
\end{align*}
If we denote $\delta_{l,i} = \frac{\partial \ll}{\partial  h_{l,i}}$ we can write this as:
\begin{align}\label{ffa}
\frac{\partial \ll}{\partial W_{l,ij}} = & \delta_{l,i} x_{l-1,j}\nonumber \\
\delta_{l,i} = &\sigma'(h_{l,i}) \sum_{k=1}^{n_{l+1}}  \delta_{l+1,k} W_{l+1,ki}. \nonumber \\
\text{or} \quad \quad \quad \delta_l = &\sigma'(h_l) W_{l+1}^t \delta_{l+1},
\end{align}
where $\sigma'(h_l)$ is the diagonal matrix with entries $\sigma'(h_{l,i})$ on the diagonal.
So we see that the update to the synaptic weight $W_{l,ij}$ is the product of the {\it feedback} activity at unit $i$ of layer $l$ 
denoted by $\delta_{l,i}$, also called the {\it error signal}, and the input activity from unit $j$ of layer $l-1$. And the feedback activity (error signal) 
of layer $l$
is computed in terms of the feedforward weights connecting unit $i$ in layer $l$ to all the units in layer $l+1$. This is the symmetry
problem. 

We now separate the feedforward weights from the feedback weights. Let $R_{l+1,ik}$ be the feedback weight connecting unit $k$ of layer $l+1$ to unit
$i$ of layer $l$. The the second equation in \eqref{ffa} becomes:
$$\delta_{l,i} = \sigma'(h_{l,i}) \sum_{k=1}^{n_{l+1}}  \delta_{l+1,k} R_{l+1,ik}.$$
If $R=W^t$ we get the original back-propagation update.
We illustrate the general updating scheme computation in figure \ref{backprop}.

In \cite{Lilli} the values of $R$ are held fixed at some random initial value, which we denote {\it fixed random feedback} (FRFB).
In contrast, in our proposal, since $R_{l+1,ik}$ connects the same units as $W_{l+1,ki}$ it experiences
the same pre and post-synaptic activity and so will be updated by the same Hebbian increment -
$ \delta_{l+1,k} x_{l,i}.$ 
We call this method {\it updated
random feedback} - URFB.
If the initial values of $R_{l,ik}$ are the same as $W_{l,ki}$
then equality will hold throughout the update iterations resulting
in a symmetric system performing precise back-propagation. 
This is the proposal in \cite{zipser1990}. We experiment 
with different initializations, so that the updates are not performing back-propagation,
even in the long run after many iterations the weights are not equal, although their correlation increases.
We show that classification rates remain very close to those obtained by
back-propagation.
In addition, in order to increase the plausibility of the model we also experiment with sparsifying the
feedforward and feedback connections by randomly fixing half of each set of weights at 0.

\begin{figure}[h]
\center
\includegraphics[height=3in]{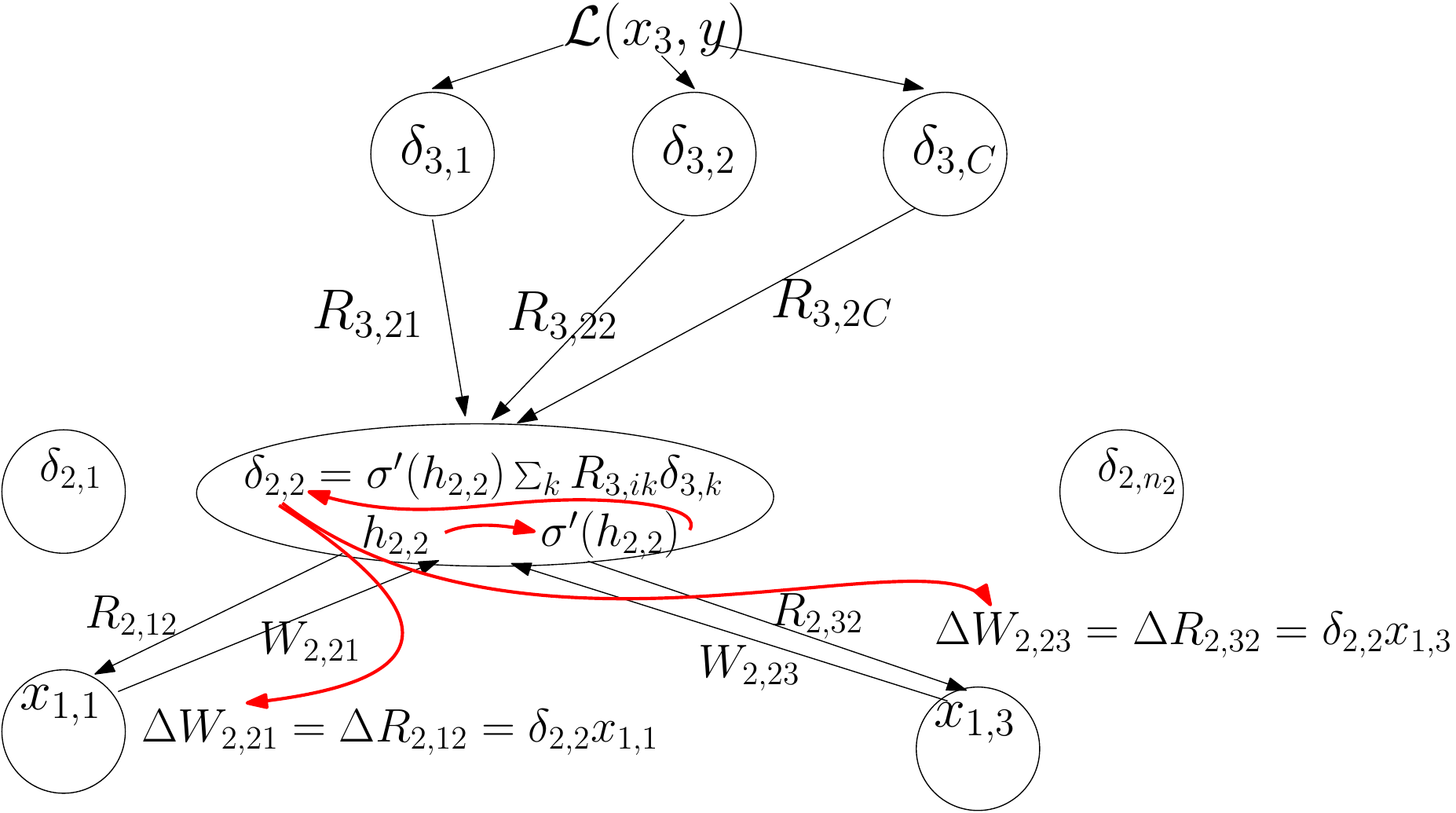}
\caption{The feedback signals $\delta_{3,k}$ from layer $3$ are combined linearly and
then multiplied by $\si'(h_{2,2})$ to produce the feedback signal $\delta_{2,2}$. Then the update to the feedforward weights coming into
unit  $(2,2)$ and feedback weights coming out of that unit is computed.The red arrows
indicate the order of computation.}\label{backprop}
\end{figure}

{\bf Remark 1:}
It is important  to note that the feedback activity $\delta_{l,i}$ replaces the feedforward activity $x_{l,i}$ and
needs to be computed before the update of the feedforward weights
feeding into unit $i$ and the feedback weights feeding out of that unit, but using the original {\it feedforward}
activity $x_{l-1,i}$ of the units in layer $l-1$. This requires a very rigid sequencing of the algorithm from top to bottom. 

{\bf Remark 2:}  The feedback signal propagates by computing a linear combination of the feedback
signals in the higher layers, but is then multiplied by the term $\sigma'(h_{l,i})$.
To simplify as much as possible we have employed a non-linearity $\sigma$ of the form
$$\si(h) = \max(-1,\min(1,h)),$$ which is simply a saturated linear function at thresholds
$-1$ and $1$,
and
$\si'(h) = 1$ if  $|h| \le 1$ and   0 otherwise.
Thus the feedback activity $\delta_{l,i}$ is the linear combination of the feedback activities $\delta_{l+1,k}$ in the layer above
unless 
\begin{equation}\label{delthresh}
| h_{l,i} | \ge 1, \text{or } |x_{l,i}|=1.
\end{equation}
i.e. bottom-up input $h_{l,i}$ is too high or too low, in which case $\delta_{l,i}=0$. 
A local network to compute this thresholding is described in Appendix 1.
The computation of the top level derivative $\delta_{L,i} = \partial \ll / \partial h_{L,i}$ will be discussed
in the next section.

\subsection{Loss function}\label{loss}

The softmax loss commonly used in deep learning defines the probability of each output class as a function of the
activities $x_{L,i}$ as follows:
$$ \text{softmax}(x_L)_c = p_c = \frac{e^{x_{L,c}}}{\sum_{i=1}^C e^{x_{L,i}}}, c=1,\ldots,C.$$
The loss computes the negative log-likelihood of these probabilities:
$$\ll(x_L,y) = -\sum_{i=1}^C x_{L,i} y_i + \log \sum_{i=1}^C e^{x_{L,i}},$$
where $y_c = 1$ if the class of the input is $c$ and $y_i=0, i \ne c$.
Thus the initial feedback signal is:
$$ \delta_{L,i} = \frac{\partial \ll(x_L,y)}{\partial x_{L,i}} = y_i-p_i.$$
This requires the computation of the softmax function $p_i$, which involves the activity of all other units, as
well as exponentiations and ratios. 

The classification loss  function used here is motivated by the hinge loss used in standard linear SVMs.
 In the simplest case of a
 two class problem we code the two classes as a scalar $y=\pm 1$ and use only one output unit $x_L$.
 Classification is based on the sign of $x_L$.
 The loss is given by
\begin{equation*}
\ll(x_L,y)=\max(1-x_Ly,0).
\end{equation*}
The derivative of this loss with respect to $x_L$, is simply
$$\frac{\partial \ll}{\partial x_L}= \begin{cases} -y & \text{if } y \cdot x_L \le 1 \\
0 & \text{otherwise}.
\end{cases}
$$
The idea is to that the output $x_L$ should have the same sign as $y$ and be sufficiently
large in magnitude. Once the output achieves that, there is no need to change it and the loss is zero.

Writing $x_L = W^t x_{L-1}$, this yields the perceptron learning rule {\it with margin} (see  \cite{srebro}):
\begin{equation*}
\frac{\partial \ll}{\partial W_i} = \begin{cases} -x_{L-1,i} & \mbox{if} \ y=1 \ \mbox{and} \ W^t x_0 \le 1 \\
x_{L-1,i} & \mbox{if} \ y=-1 \ \mbox{and} \ W^t x_{L-1}\ge -1 \\
0 & \mbox{otherwise}
\end{cases},
\end{equation*}

If we think of the supervised signal as activating the output unit with $\delta_L=+1$ for one class and $\delta_L=-1$ for the other, 
unless the input is already of the correct sign and of magnitude greater than 1, then
$\delta_L = -\partial \ll/\partial x_L$.
The update rule can be rewritten as $W_i \leftarrow W_i + \eta \Delta W_i$
where $\Delta W_i = \delta_L\cdot x_{L-1,i}$ if  $x_L= W^t x_{L-1}$ satisfies $\delta_L x_L  \le 1.$
In other words if the output $x_L$ has the correct sign by more than the margin of {\bf 1} then no update occurs,
otherwise the weight is updated by the product of the target unit activity and the input unit activity.
In that sense the update rule is Hebbian, except for shut down of the update when $x_L$ is `sufficiently correct'.

One might ask why not use the unconstrained Hebbian update  $\Delta W_i =\eta \delta_L  x_{L-1,i}$, which corresponds to a loss that computes the inner product
of $y$ and $x$.  Unconstrained
maximization of the inner product can yield over fitting in the presence of particularly large values of some of the coordinates of $x$ and  create an imbalance between the
two classes if their input feature distribution is very different. This becomes all the more important with multiple classes, which we discuss next.

For multiple classes we generalize hinge loss as follows. 
Assume as before $C$ output units $x_{L,1},\ldots,x_{L,C}$. For an example $x$, of class $c$ 
define the loss
\begin{equation}\label{cost}
\ll(x_L,y) =  \max(1-x_{L,c},0) + \mu \sum_{i \ne c} \max(1+x_{L,i},0).
\end{equation}
where $\mu$ is some
balancing factor. 
The derivative has the form:
\begin{equation}\label{delta_out}
\frac{\partial{\ll(x_L,y)}}{\partial x_{L,i}} = \begin{cases} -1 & \text{ if } i=c  \text{ and } x_{L,i} \le 1 \\
\mu &   \text{ if }  i \ne c \text{ and } x_{L,i} \ge -1 \\
0 & \text{ otherwise.}
\end{cases}
\end{equation}
Henceforth we will set $\delta_{L,i} = - \partial{\ll(x_L,y)}/{\partial x_{L,i}}.$ Substituting
the feedback signal $\delta_{L,i}$ for the feedforward signal $x_{L,i}$ at the top layer
has the following simple form:
\begin{equation}\label{delta_outa}
\delta_{L,i} = \begin{cases} 1 & \text{ if } i=c  \text{ and } x_{L,i} \le 1 \\
-\mu &   \text{ if }  i \ne c \text{ and } x_{L,i} \ge -1 \\
0 & \text{ otherwise.}
\end{cases}
\end{equation}
and equation \eqref{ffa} is then applied to compute the feedback to layer $L-1$ - $\delta_{L-1}$ and the
update of the weights $W_{L}, R_{L}.$
All experiments below use this rule. 

Note that $\delta_{L,i}$ is precisely the target signal, {\it except} when the feedforward signal has the correct
value - greater than 1 if $i=c$ (the correct class) and less than $-\mu$ for $i \ne c$ (the wrong class). 
This error signal only depends on the target value and input to unit $i$, no information is needed regarding
the activity of other units.
One can ask whether a neuron can produce such an output, which depends both on the exterior teaching signal and on the feedforward
activity. In Appendix 1 we propose a local network that can perform this computation.

This loss produces  the well known one-versus-all method for multi-class SVM's (see for example \cite{ova}),
where for each class $c$ a two class  SVM is trained for class $c$ against all the rest lumped together as one class.
Classification is based on the maximum output of the $C$  classifiers. 
Each unit $x_{L,c}$ can be viewed as a classifier of class $c$ against all the rest. When an example
of class $c$ is presented it updates the weights to obtain a more positive output, when an example
of any class other than $c$ is presented it updates the weights to obtain a more negative output.
Other global multiclass losses for SVM's can be found in \cite{ova}.
In \cite{amit-vr,amit-walker} a network of binary neurons with discrete synapses was described that implements this learning
rule to update connections between discrete neurons in the input and output layers and with positive synapses.
Each class was represented by multiple neurons in the output layer. Thus classification was achieved through recurrent dynamics in the output layer,
where the class with most activated units maintained sustained activity, whereas activity in the units corresponding to other classes died out.

\section{Experiments}\label{exp}
\begin{figure}[h]
\center
\includegraphics[width=3in]{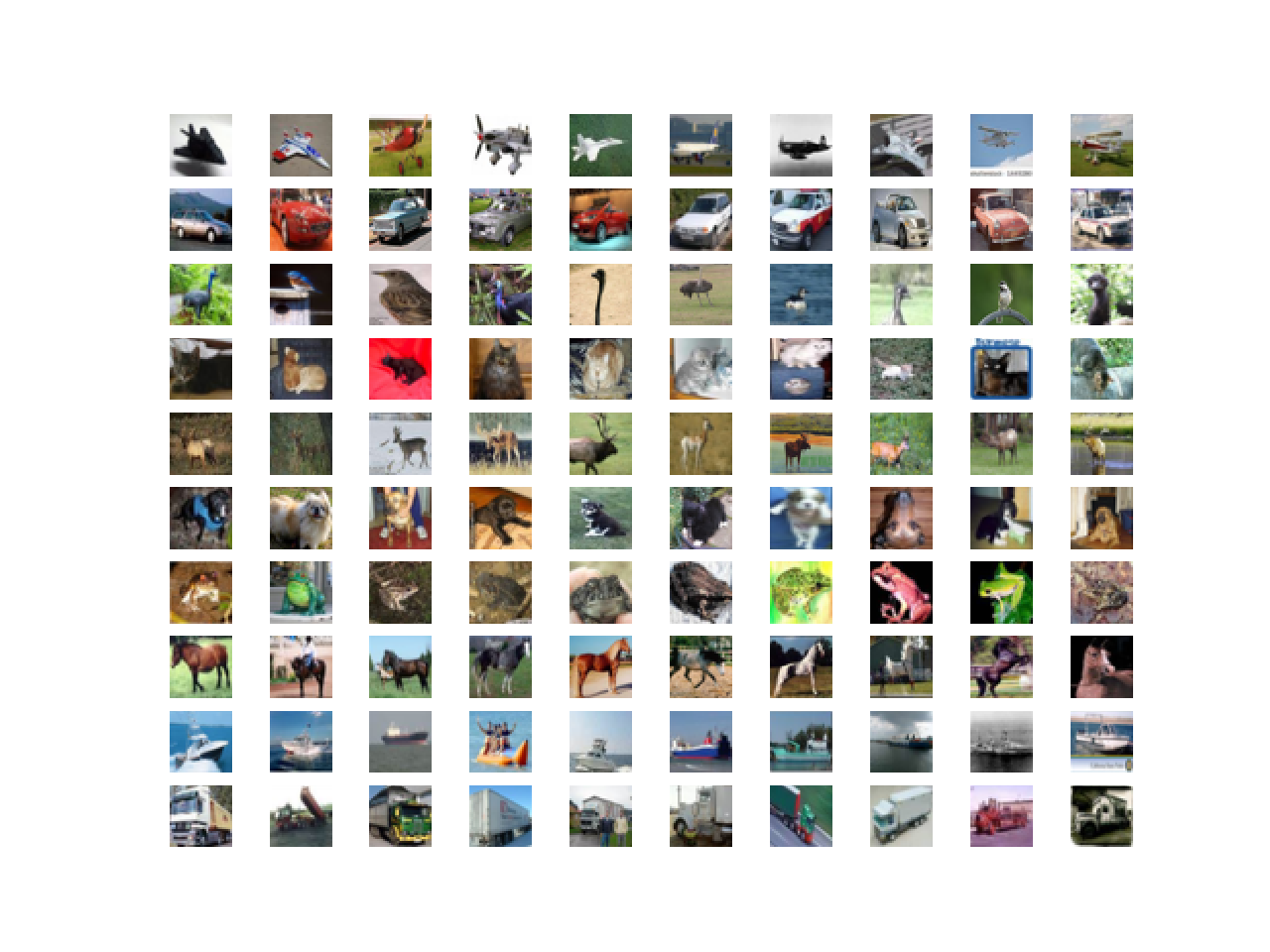}%
\includegraphics[width=3in]{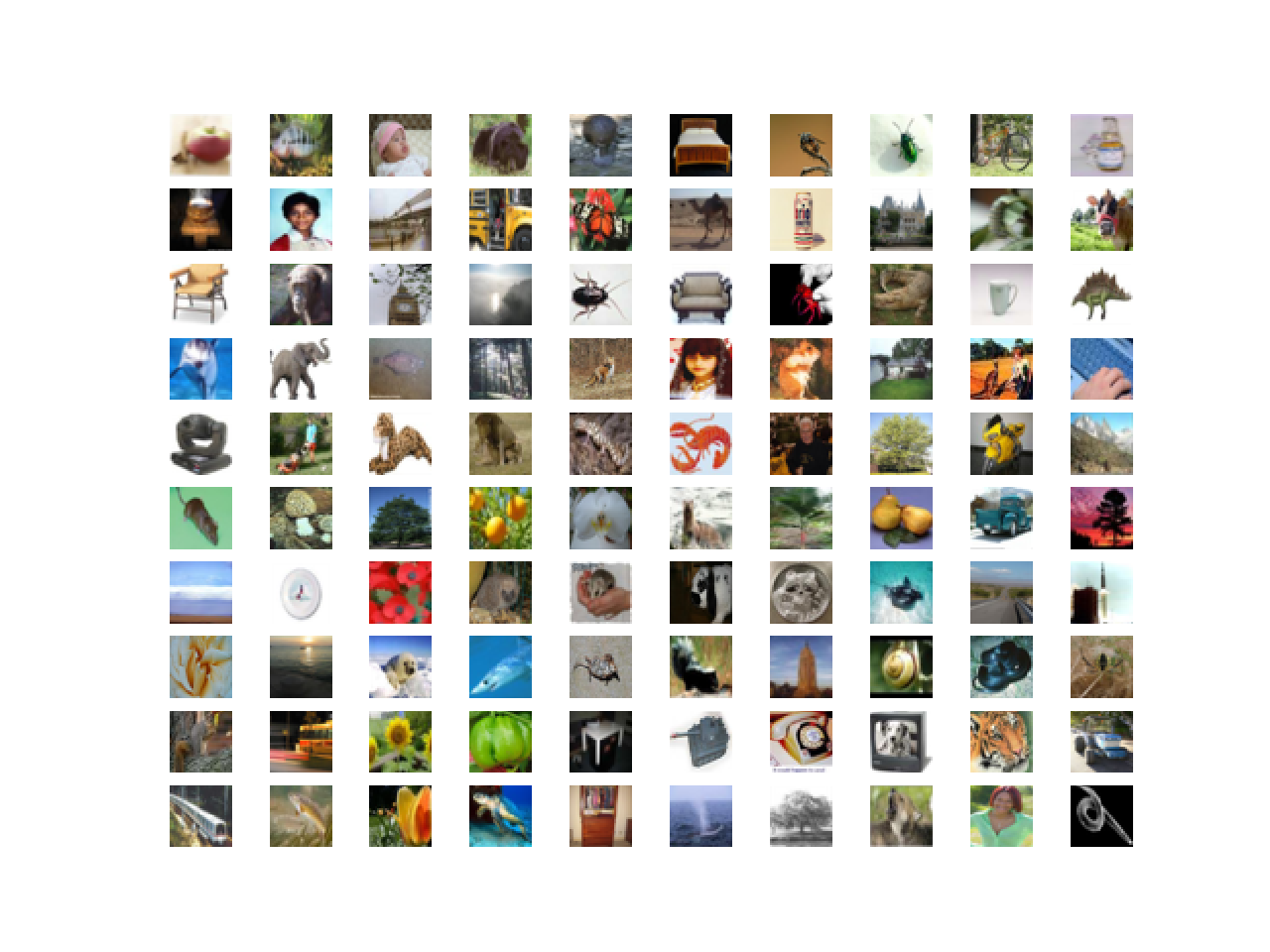}
\caption{Left: Each row showing 10 images from one of the 10 cifar10 classes.
Right: One image from each of the 100 classes in cifar100}\label{cifar}
\end{figure}

We report a number of experiments comparing the updated (URFB) to the fixed feedback matrix (FRFB)
and comparing the multi-class hinge loss function to the cross-entropy with softmax loss. 
We restrict ourselves to image data. Since it is quite easy to obtain good results
with the widely used MNIST handwritten data set \cite{mnist} we focus on two more challenging
data sets called CIFAR10 and CIFAR100 \cite{cifar}. Each dataset contains 32x32 color images from 10 classes
for the first and 100 classes for the second. The classes are broadly defined so that the category
bird will contain multiple bird types at many different angles and scales. Some sample images
are show in figure \ref{cifar}.
Each data set has 50000 training images and 10000 test images.
 An even more challenging
data set is IMAGENET \url{http://www.image-net.org/}, which contains thousands of classes. We do not experiment with this
dataset as the images are much larger and the training set is much larger so that training is time consuming.

There are a number of benchmark network architectures that have been developed over the past
decade with good results on these datasets, see (\cite{vgg16,alexnet,resnet}). These networks
are very deep and employ a variety of methods to accelerate convergence, such as adaptive time-steps
and batch normalization. These improvements involve steps that are not easily modeled as neural computations. For that reason we restrict our learning method to the simplest form of gradient descent
with a fixed time step and no normalization. We do not perform any pre-processing of the input data,
nor do we employ any methods of data augmentation to improve classification results.
All our weights are initialized based on the method described in \cite{glorot}. Weights are uniformly drawn between
$[-b_l,b_l]$ where $b_l$ is a function of the number of incoming and outgoing connections to a unit in layer $l$.

In the experiments we demonstrate the following:
\begin{enumerate}
\item With regular back-propagation (BP) hinge loss performs slightly worse than the softmax loss but results are comparable.
\item For shallow networks URFB performs somewhat better then FRFB but mainly converges faster.
It never performs as well as BP but is close.
\item For deeper networks URFB again performs close to BP but FRFB performance degrades significantly.
\item With locally connected - untied - layers instead of convolutional layers results are slightly worse overall but
the relationship between the different methods is maintained.
\item In URFB the feedback weights are never the same as the feedforward weights, although
the correlation between the two sets of weights increases with iteration. 
\item Even in initial iterations, when the weights are far from being aligned, training and validation error rates decrease at similar rates to back propagation.
\end{enumerate}

We first experiment with a shallow network with only two hidden layers, one convolutional and one fully connected.
\noindent
\small{
\begin{verbatim}
simpnet: Conv 32 5x5; Maxpool 3; Drop .8; Full 500; Drop .3; Output 
\end{verbatim}
}
The notation \texttt{Conv 32 5x5} means that we have 32 - 5x5 filters, each applied as a convolution
to the input images, producing 32 output arrays of dimension 32x32. \texttt{Maxpool 3} means that at each pixel
the maximal value in the 3x3 window centered at that pixel is substituted for the original value (padding with 0's outside the grid), 
in each of the 32 output arrays, and then only every second pixel is recorded producing 32 arrays of size 16x16.
\text{Drop .8} means that for each training batch, a random subset of 80\% of the pixels in each array are set to 0 so that
no update occurs to the outgoing synaptic weights. This step was introduced in \cite{dropout}  as a way to avoid
overfitting to the training set. It is also attractive as a model for biological learning as clearly not all synapses will
update at each iteration.
\texttt{Full 500} means a layer with 500 units, each connected
to every unit in the previous layer. 

\begin{figure}[h]
\center
\includegraphics[width=.7\textwidth]{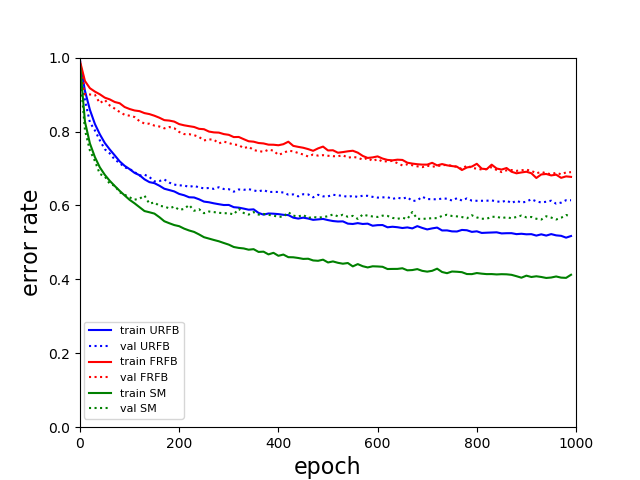}
\caption{Evolution of error rates for \texttt{simpnet} as a function of epochs. 
Solid lines training error, dotted lines validation error. Green - BP, Blue - URFB, Red - FRFB
\label{simple}}
\end{figure}

The \texttt{Output} layer has $C$ output units corresponding to each class.
We use the  saturated linearity $\sigma(x)  = \min(\max(x,-1),1)$ and the hinge loss function as given in \eqref{cost}.
The update is a simple SGD with a fixed time step of .1, and the network is trained for 1000 epochs
with batch-size of 500. We make a point to avoid any adaptive normalization
layers as these require a complex gradient that is not amenable to simple
neural computations. We avoid the more sophisticated time step adaptations which depend on previous updates and some normalizations,
which again do not seem amenable to simple neural computations.

The three parameters we adjusted were the time step and two drop out rates. 
We experimented with time-steps .01, .1 and 1.0 for the \texttt{simpnet} and found the best behavior on a held out validation set of 5000 samples
was with the value .1. We kept this value for all further experiments. We had two dropout layers in each network. One between convolutional
layers and one before the output layer. The values were adjusted by running a few tens of iterations and making sure the validation loss
was closely tracking the training loss.  

We also experiment with pruning the forward and backward connections randomly by 50\%. 
In other words half of these connections are randomly set to 0. Error rates for CIFAR10 and CIFAR100 datasets are shown in figure \ref{T1}. 
We  note
 that the use of the  multi-class hinge loss  leads to only a small  loss in accuracy relative to softmax. All experiments
 with random feedback are performed with the hinge loss.
For CIFAR10 the difference between R fixed - FRFB - and R updated - URFB - is small, but 
becomes more significant when connectivity is reduced to 50\% and with the CIFAR100 database.

\begin{figure}[h]
\includegraphics[width=.5\textwidth]{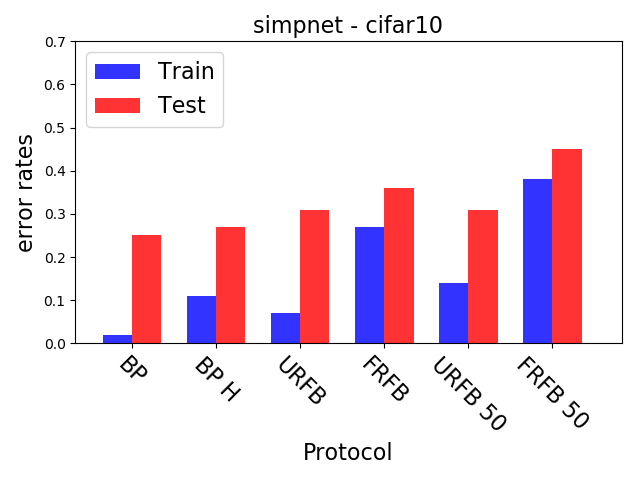}
\includegraphics[width=.5\textwidth]{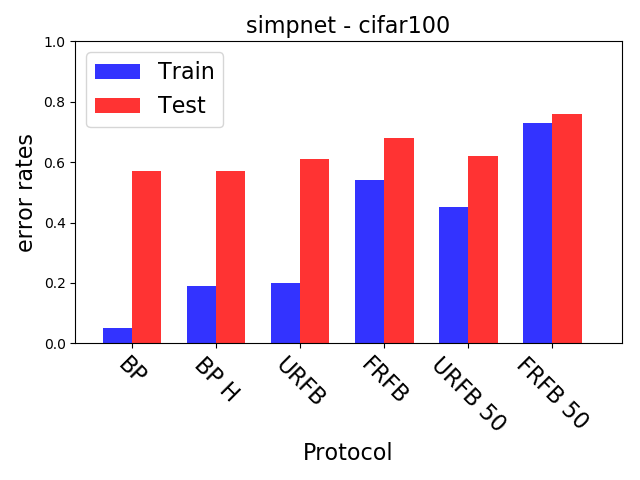}
\caption{\small Error rates for \texttt{simple} network with different update protocols and different losses.
Left: CIFAR10, Right: CIFAR100.
BP - back-propagation with softmax and cross entropy loss, BP-H - back propagation with
hinge loss, all other protocols use the hinge loss as well.
URFB - Updated random feedback and, FRFB - Fixed random feedback. 
50\% refers to random  connectivity. 
\label{T1}}
\end{figure}
Note that in the simple network the only layer propagating back an error signal is the fully connected layer.
The first layer, which is convolutional, does not need to back-propagate an error. The evolution of error rates for
the different protocols as a function of iteration can be seen in figure \ref{simple}.

We experiment with a deep network with multiple convolutional layers, and observe an even larger
difference between R fixed and R  updated. With the deep network FRFB performs very poorly.
The deep architecture
is given here.

\small{
\begin{verbatim}
deepnet: Conv 32 5x5; Maxpool 3; Conv 32 3x3; Conv 32 3x3; Maxpool 3; Drop .8;
         Conv 32 3x3; Conv 32 3x3; Maxpool 3; Drop .3; Full 500; Output 
\end{verbatim}
}
Finally we try an even deeper network with residual layers as in \cite{resnet}.
This means that after every pair of consecutive convolutional layers at the same
resolution we introduce a layer that adds the two previous layers with no trainable
parameters. This architecture was found to yield improved results 
on a variety of datasets. 
\small{
\begin{verbatim}
deepernet: conv 16 3x3; conv 16 3x3; SUM; conv 32 3x3; conv 32 3x3; 
        SUM; maxpool 3; drop .5; conv 64 3x3; conv 64 3x3; SUM; maxpool 3;
        conv 128 3x3; conv 128 3x3; SUM; maxpool 3; drop .8;
        full conn. 500; output 
\end{verbatim}
}%

\begin{figure}[h]
\includegraphics[width=.5\textwidth]{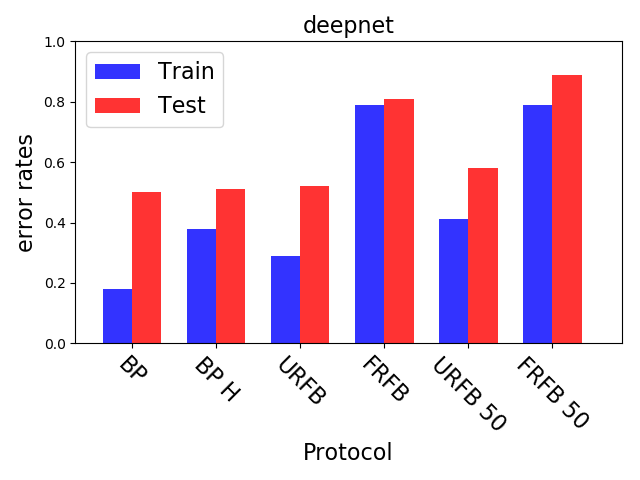}
\includegraphics[width=.5\textwidth]{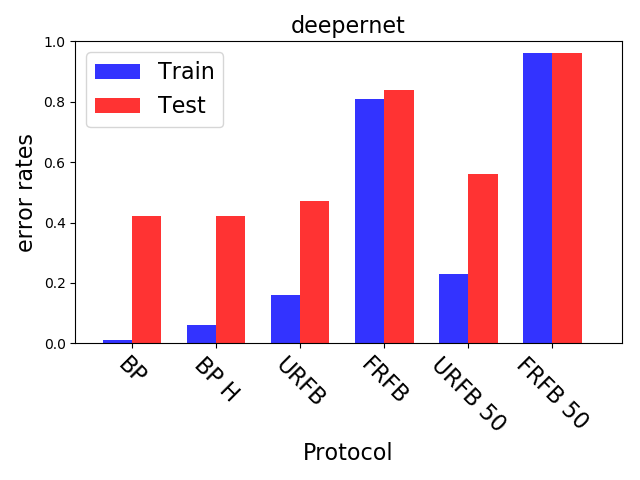}
\caption{\small Error rates for the \texttt{deepnet} -left and \texttt{deepernet} -right.
BP - back-propagation with softmax and cross entropy loss, BP-H - back propagation with
hinge loss, all other protocols use the hinge loss as well.
URFB - Updated random feedback and, FRFB - Fixed random feedback. 
50\% refers to random  connectivity.  \label{deepest}
}
\end{figure}

We see in figure \ref{deepest} that for the default BP with softmax or hinge loss the
error rate decreases from 50\% with \texttt{deepnet} to 42\% with \texttt{deepernet}. URFB also shows a decrease in error between
\texttt{deepnet} and \texttt{deepernet} and again
FRFB performs very poorly.
The evolution of error rates for
the different protocols as a function of iteration can be seen in figure \ref{fig-deepest}.
\begin{figure}[h]
\center
\includegraphics[width=.7\textwidth]{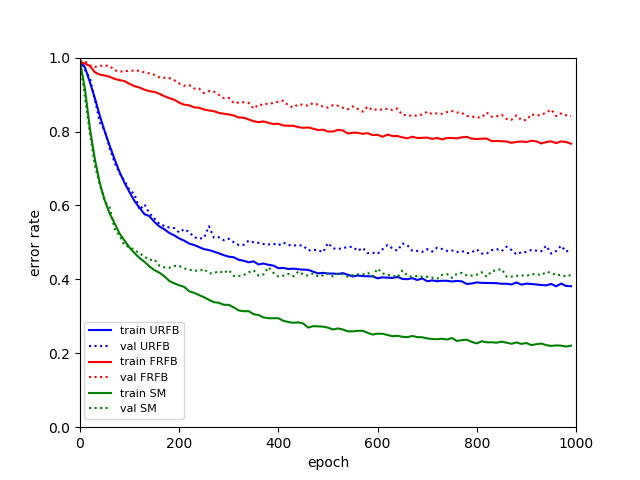}
\caption{Evolution of error rates for \texttt{deepernet} as a function of epochs. 
Solid lines training error, dotted lines validation error. Green - BP, Blue - URFB, Red - FRFB. \label{fig-deepest}}
\end{figure}

\subsection{Untying the convolutional layers - locally connected layers}
We explore `untied' local connectivities determined by the corresponding
convolutional layer. These are also called locally connected layers \cite{Lillicrap3}.
A convolution corresponds to multiplication by a sparse matrix where the entry values are repeated in each row, but with some displacement.
This again is not plausible because it assumes identical weights across a retinotopic layer.
Furthermore the back-propagation
update of a particular weight in a convolutional layer  computes the {\it sum} of all products $\sum_i \delta_{l,i} x_{l-1,i+k}$,
where $i$ represents  locations on the grid and $k$ is a fixed displacement. So, it assumes that each one of the identical weights
is updated by information summed across the entire layer. 

To avoid these issues with biological plausibility we
instead assume each of the entries of the sparse matrix is updated separately with the corresponding product $\delta_{l,i} x_{l-1,i+k}$.
Only non-zero elements of the sparse matrix, that correspond to connections implied by the convolutional operation are updated.
This is implemented using tensorflow sparse tensor operations, and is significantly slower and requires more memory than the
ordinary convolutional layers.
The error rates are similar to those with the original
convolutional layers even with the deeper networks. 
In figure \ref{unty} for CIFAR10,  we  show a comparison of error rates
between networks with convolutional layers to networks with corresponding untied layers for
the different training protocols. We show comparisons for \texttt{simpnet} 
and \texttt{deepnet\_s} defined below.

\begin{figure}[h]
\begin{center}
\includegraphics[width=1in]{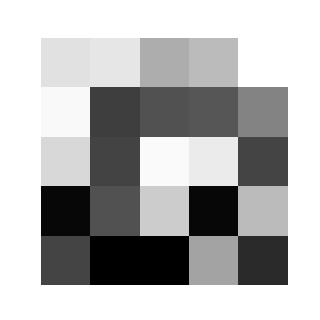}
\includegraphics[width=1in]{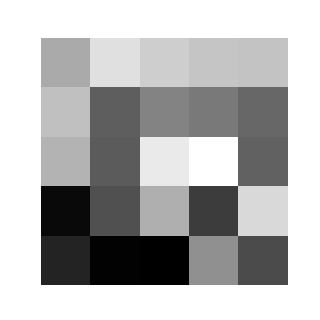}
\includegraphics[width=1in]{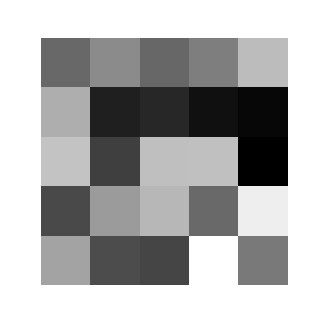}
\includegraphics[width=1in]{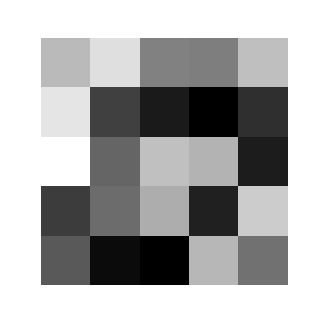}

\includegraphics[width=1in]{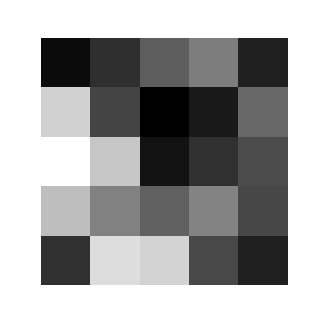}
\includegraphics[width=1in]{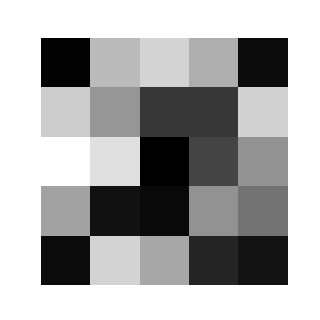}
\includegraphics[width=1in]{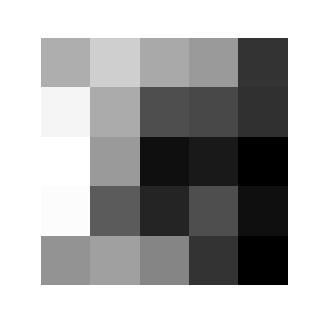}
\includegraphics[width=1in]{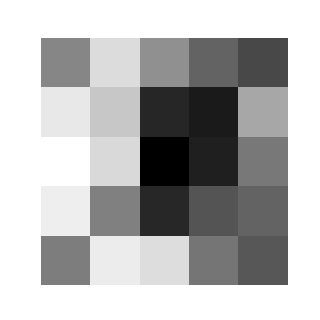}
\end{center}
\caption{Corresponding filters extracted from the sparse connectivity matrix at 4 different locations
on the 32x32 grid.\label{filt}}
\end{figure}

Despite the fact that the weights are updated 
without being tied across space, the final connectivity matrix retains a strong spatial homogeneity. In other words at each location of the output layer 
one can restructure the weights to a filter and inspect how similar these filters are across locations. We presume that this is due to the fact that in the data local
structures are consistent across space.
In figure \ref{filt} we show a couple of these 5x5 filters across four different locations in the 32x32 grid in the trained \texttt{simpnet}.
We see that even after 1000 iterations there is significant similarity in the structure of the filters despite the fact that they were updated independently for each location.

We also experiment with a deeper network:
\small{
\begin{verbatim}
deepnet_s: conv 16 3x3; conv 16 3x3;
	SUM;maxpool 3, stride 3; drop .5; 
        conv 64 3x3; conv 64 3x3; SUM; maxpool 2, stride 2;
        conv 64 2x2; conv 64 2x2; SUM; maxpool 2, stride 2; drop .5;
        full conn. 500; output 
\end{verbatim}
}%
Here we could not run all convolutional
layers as untied layers due to memory constraints on our GPUs. Instead we ran the network for 100 epochs
with the regular convolutional layers,  then we froze the first
layer and retrained the remaining layers from scratch using the untied architecture, see figure \ref{unty}. This would mimic a situation where the first
convolutional layer perhaps corresponding to V1 has connections that are predetermined and not subject to synaptic modifications.
Once more, we see that the untied layers with URFB reach error rates similar to those of the regular convolutional layers with standard gradient descent. And again,
we observe that with a deeper network FRFB performance  is much worse.

\begin{figure}[h]
\center
\includegraphics[width=4in]{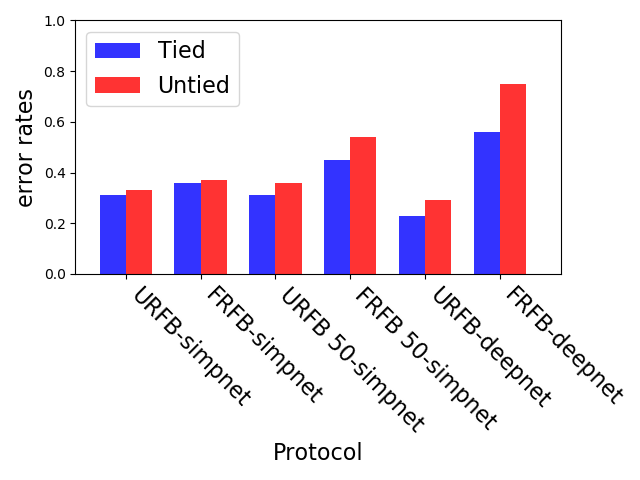}
\caption{\small Experiments with untying the convolutional layers on  \texttt{simpnet} and \texttt{deepnet\_s}. 
Blue - convolutional layers (tied), Red - untied.
\label{unty}}
\end{figure}

\subsection{Weight alignment}
One of the claims in \cite{Lilli} is that the network gradually  aligns the updated feedforward weights to the fixed feedback weights.
In figure \ref{corr_simp} we show the evolution of the correlations between the feedforward weights $W_l$ and $R_l$ for \texttt{simpnet}.
Recall that the layer with highest index is the output layer and typically reaches
high correlations in both URFB and FRFB.
We see, however, that the alignment is much stronger for the URFB. 
Note that when weights are highly correlated the network is effectively implementing
back-propagation. 

In figure \ref{corr_deep} we again show the evolution of the correlations between $W_l,R_l$ for the seven updated
layers of the deeper network \texttt{deepnet\_s}. Note that for some but not all layers the final correlations are very close to one. However the training loss
and error rates change very rapidly in the initial iterations when the correlations are very low. Interestingly the correlation levels
are not a monotone function of layer depth.

\begin{figure}[h]
\center
\includegraphics[height=2.5in]{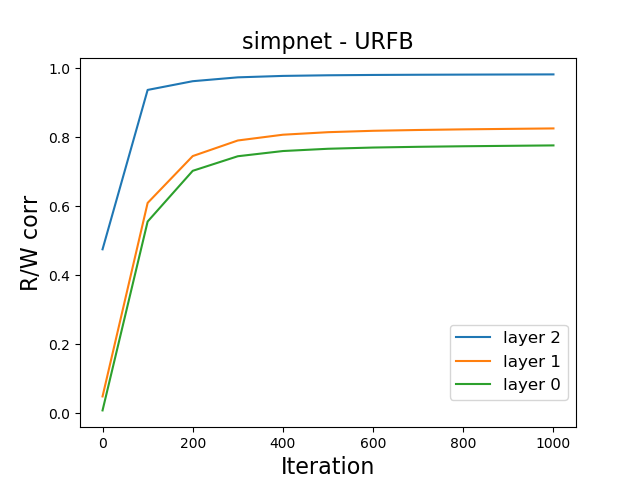}%
\includegraphics[height=2.5in]{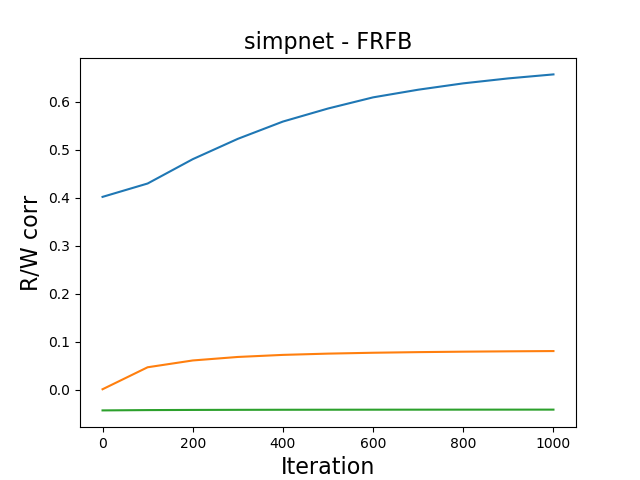}
\caption{Correlation between $W_l$ and $R_l$ for the three layers in \texttt{simpnet}. Left:
 URFB, Right: FRFB}\label{corr_simp}
 \end{figure}
\begin{figure}[h]
\center
\includegraphics[height=2.5in]{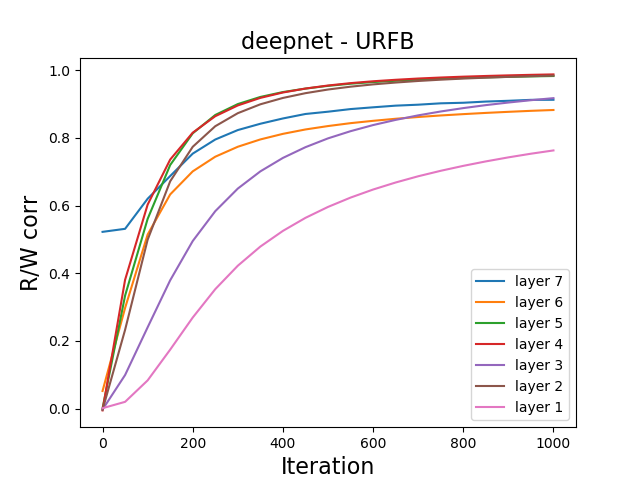}%
\includegraphics[height=2.5in]{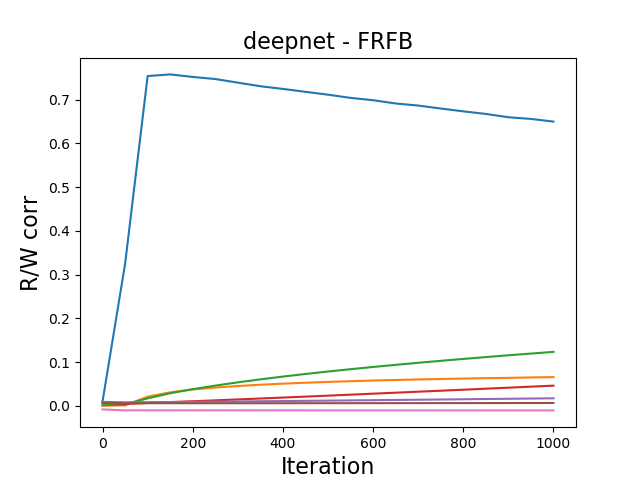}
\caption{Correlation between $W_l$ and $R_l$ for the seven updated layers in \texttt{deepnet\_s}. Left:
 URFB, Right: FRFB}\label{corr_deep}
 \end{figure}

\section{Mathematical analysis of updated random feedback}\label{theory}

The mathematical analysis closely follows the methods developed in \cite{ganguli}
and thus focuses on linear networks, i.e. $\si(x) =x$ and a simple quadratic loss.
We start with  a simple two layer network.

Let the input $x \in \R^{n_0}$, and the output $y=W_2 W_1 x \in \R^{n_2}$ with weights $W_1 \in \R^{n_1 \times n_0}, W_2 \in \R^{n_2 \times n_1}$. If $X$ is the $n_0 \times N$ matrix of input data
and $Y$ the $n_2 \times N$ of output data the goal is to minimize 
$$ C(\wa,\wb) = | Y - W_2W_1 X |^2.$$
We write $T = Y X^t \in \R^{n_2 \times n_0}$, and assume that $XX^t=I$, namely the input
coordinates are uncorrelated. The gradient of $L$ with respect to $\wa$ and $\wb$ yield the
following gradient descent ODE's, which corresponds to regular back-propagation:

\begin{align*}
\dot\wb = & (T-\wb\wa) \wat  \nonumber \\
\dot\wa = & \wbt (T-\wb \wa),
\end{align*}
with some initial condition $\wa(0), \wb(0).$
If we implement the FRFB or URFB described above we get the following three equations:

\begin{align}\label{mbp}
\dot\wb = & (T-\wb\wa) \wat  \nonumber \\
\dot\wa = & \rb (T-\wb \wa) \\
 \dot\rb = &\epsilon \wa (T-\wb\wa)^t, \nonumber
\end{align}
where $\rb \in \R^{n_1 \times n_2}$ and $\epsilon$ is a parameter. 
Setting $\eps=0$ corresponds to FRFB, as there is no modification of the matrix $R$. 
The URFB corresponds to $\eps=1.$
Our goal is to show that the larger $\epsilon$ the faster the convergence of the error to 0.

To simplify the analysis of \eqref{mbp} we assume $\wa(0)=\wb(0)=0$ and $\rb(0)$ is random.
Then $\rb = \rb(0)+\eps \wbt$ and the system reduces to 
\begin{align}\label{mbpa}
\dot\wb = & (T-\wb\wa) \wat  \nonumber \\
\dot\wa = & (\rb(0)+\eps \wbt) (T-\wb \wa).
\end{align}

For deeper networks, and again assuming the $W_l$ matrices are initialized at 0,
we have the following equations for  URFB:

\begin{align}\label{mk}
\dot W_k &=  E W_{1}^t \cdots W_{k-1}^t \nonumber \\
& \vdots \nonumber \\
\dot W_{i} &=  (R_{i+1}(0)+\eps W_{i+1}^t) \cdots (R_{k}(0)+\eps W_{k}^t) E W_{1}^t \cdots W_{i-1}^t  \nonumber\\
&  \vdots \nonumber \\
\dot W_{1} &=  (R_{2}(0)+\eps W_{2}^t) \cdots (R_{k}(0)+\eps W_{k}^t) E,
\end{align}
where $E=T-W_k \cdots W_1, \quad T \in \R^{n_k \times n_0}$ and $W_i \in \R^{n_i \times n_{i-1}}, i=1,\ldots,k$.
Again our goal is to show that as $\epsilon$ increases from 0 to 1, the error given by $e=tr(E^tE)$ converges faster to 0.

The precise statements of the  results and the proofs can be found in Appendix 2. Here we show
through a simulation that convergence is indeed faster as $\epsilon$ increases from $\eps=0$ (FRFB) to $\eps=1$
(URFB).

\subsection{Simulation}\label{sim}

We simulated the following setting. 
An input layer of dimension 40, two intermediate layers of dimension 100 and an output layer of dimension 10. We assume $X=I_{40}$ so that
$T= W^*_1 W^*_2 W^*_3$ with
$W^*_1 \in R^{40 \times 100}, W^*_2 \in R^{100 \times 100}, W^*_3 \in R^{100 
\times 10}.$
We choose the $W^*_i$ to have random independent normal entries with sd=.2.
We then initialize the three matrices randomly as $W_i(0), i=1,2,3$ to run regular back propagation.
For comparison we initialize $W_i(0)=0$ and initialize $R_i(0)$ randomly.
We run the differential equations with $\eps=0,.25,.5,1.$, where $\eps=0$ corresponds to $FRFB$ and 
$\eps=1$ to URFB. We run 1000 iterations until all 5 algorithms have negligible error. We see the results in figure \ref{simf}. In the first
row, for 3 different runs we show the log-error as a function of iteration, and clearly 
convergence rate increases with $\epsilon$.
In the three rows below that we show the evolution of the correlation of $W_l$ and $R_l^t$
with the same color code. We see that for FRFB (green) the correlation of the weights feeding into the last layer 
increases to 1 but for the deeper layers that does not hold. Moreover as $\epsilon$ increases to 1 the
correlations approach higher values at each layer. The top layer always converges to a correlation very close
to 1, lower layers do not reach correlation 1., and interestingly the correlation reached in the input layer
is slightly higher than that of the middle layer. Similar non-monotonicity of the correlation was observed in the
experiments in figure \ref{corr_deep}.

\begin{figure}[h]
\begin{center}
\includegraphics[height=4in,width=\textwidth]{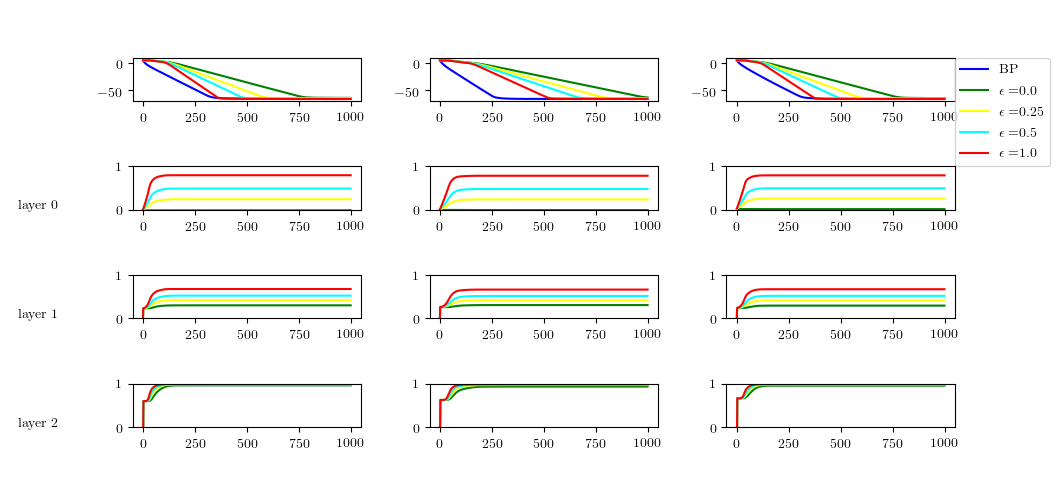}%
\end{center}
\caption{\small Top: comparison of log-error rates as a function of iteration for original BP and for four
different values of $\epsilon=0,0.25,.5,1.$
Results for three runs of the experiment. 
Last three rows, for each level of the network we show the evolution of the correlation between the $W$ and $R^t$
weights, for each of the values of $\epsilon.$}\label{simf}
\end{figure}

\section{Discussion}

The original idea proposed in \cite{zipser1990} of having separate feedback weights
undergoing the same Hebbian updates as the feedforward weights yields the original back-propagation algorithm
if the feedforward and feedback weights are initialized with the same values. We have shown that even when these
weights are initialized differently the performance of the algorithm is comparable to that of  back-propagation
and significantly outperforms fixed feedback weights as proposed in \cite{Lilli}. The improvement over fixed feedback
weights increases with the depth of the network and is demonstrated on challenging benchmarks such as CIFAR10 and CIFAR100.
We have also shown that in the long run the feedforward and feedback weights increase their alignment but the performance of 
the algorithm is comparable to back-propagation even at the initial iterations.
We have introduced a cost function whose derivatives lead to local Hebbian updates
and provided a proposal for how the associated error signal in the output layer could be implemented in a network.
We have shown theoretically,
in the linear setting, that adding the update to the feedback weights accelerates the convergence of the error
to zero. 

These contributions notwithstanding, there are still many aspects of this learning algorithm
that are far from biologically plausible. First, although we have removed the need for symmetric connections,
we have maintained a symmetric update rule, in that the update of a feedback and feedforward connection connecting
two units is the same. To use the formulation
in \cite{gerstner} a typical Hebbian update has the form $\Delta W = f(x_{pre})g(x_{post}),$ where $f,g$ 
are typically {\it not} the same function, however in our setting both $f$ and $g$ are linear which yields a symmetric Hebbian
update. In \cite{burbank} it is shown that a mirrored version of STDP could produce this type of symmetric update. Whether this
is actually biologically realistic is still an open question. 

Another important issue is the timing of the feedforward and feedback
weight updates that needs to be very tightly controlled. The update of the feedforward and feedback connections between
layer $l$ and $l+1$ requires the feedback signal to layer $l+1$ to have replaced the the feedforward signal in all its units,
while the feedforward signal is maintained in layer $l$. This issue is discussed in detail in \cite{Lillicrap2}.
They propose a neural model with several compartments. 
One that receives bottom-up or feedforward input and one that receives top-down feedback input. In a transient phase corresponding to the
feedforward processing of the network the top-down input contribution
to the neural voltage at the soma is suppressed. Then in a second phase this voltage is allowed in
and combined with the feedforward voltage contribution to enable synaptic modifications. In our proposal,
instead of combining the two voltages, the top-down voltage would replace the bottom up voltage. 
Still, in a multilayer network, this would need to be timed in such a way that the previous layer is still responding
only to the feedforward input.

An important component of the model proposed in  \cite{fbu} are the synaptic tags
that maintain the information on the firing of the pre and post-synaptic neurons allowing for a later synaptic modification
based on some reinforcement signal. This may offer a mechanism for controlling the timing
of the updates. An alternative direction of research would be to investigate the possibility of
desynchronizing the updates, i.e. making the learning process more stochastic. If images of similar classes are shown
in sequence it could be that it is not so important when the update occurs, as long as the statistics of the error signal
and the feedforward signal are the same. 

We have defined the network with neurons that have negative and positive values,
and synapses with negative and positive weights. Handling negative weights can be achieved with
properly adjusted inhibitory inputs. Handling the negative neural activity is more challenging and it would be
of interest to explore an architecture that employs only positive neural activity.
Finally we mention the issue of the training protocol. 
We assume randomly ordered presentation of data from all the classes, many hundreds of times. A more natural protocol would be to learn classes one at a time,
perhaps occasionally refreshing the memory of previously learned ones. Because our loss function is local
and updates to each class label are independent, one could potentially experiment with alternative protocols
and see if they are able to yield similar error rates.

\newpage
\noindent
{\bf Acknowledgements:} This work was supported in part by NIMH award no. R01 MH11555. I'd like to thank Nicolas Brunel, Ulises Pereira and the anonymous referees for helpful comments. 

\medskip
\noindent
{\bf Code:} Code for URFB can be found in \texttt{https://github.com/yaliamit/URFB.git}

\section*{Appendix 1: Networks for local thresholding}
We describe two local networks, one for computing the error signal $\delta_L$ (see \eqref{delta_outa}) at the top level
output units, and one to shut off the feedback signal $\delta_{l,i}$ based on \eqref{delthresh}.

\subsection*{Computing the output error signal}
We first describe a network for computing
the top level error signal  $\delta_L$, defined in equation \eqref{delta_out}, which depends on the activity of the output unit $x_L$
as well as on the target signal.
There have been many models proposed for Hebbian learning in terms of
non-linear functions of both the presynaptic activity and the state of the post-stynaptic neuron,
reviewed in detail in \cite{gerstner}.
A model for the particular dependence required here, where synaptic modification stops when the
input is sufficiently correct was proposed in \cite{fusi}.
A mechanism internal to the neuron is proposed, that
shuts off potentiation or depression of its incoming 
synapses when the input is too high or too low. However, shutting off synaptic modifications, does not manifest itself in the activity
of the neuron.  An implementation that is internal to the neuron provides
no explicit error signal, which can then be propagated to previous layers with the feedback connections.
When the network only has two layers, an input and an output, that is not an issue, but
with deeper networks we will need to use feedback connections to propagate the error signal.

as above, to avoid complications of modeling excitatory and inhibitory neurons we
 assume neurons have positive and negative firing rates and synaptic connections
are positive and negative. The main idea is to introduce a control unit $t_c$ that shuts the main unit $o_c$ off
when the input is outside the appropriate range.

Let $\delta_c$ be the activity of a neuron  associated with class $c$ in the top layer,
with input  given by $h_c = \lan W_c,x \ran$.  For simplicity we omit the index of the input layer. Assume $h_c$ always lies in the interval $[-M,M]$. Given a learning threshold $S$ (in the previous section we had $S=1$),
we want $\delta_c=1$ if class $c$ is presented and $h_c<S$, $\delta_c=-\mu$ if another class is presented and $h_c>-1$ and
$\delta_c=0$ otherwise.
Let $s_c$ be the unit providing the supervisory signal: 1 if class $c$ is being presented, -1 otherwise,
and let $t_c$ be a `control' neuron.
The input to $t_c$ is simply $o_c$ and
the full input of $o_c$ is
\begin{equation*}
H_c = h_c + 2M s_c - 2M t_c,
\end{equation*}
Also $\delta_c = \sigma_\delta(H_c), \quad t_c=\sigma_t(o_c).$

Let $\mu \le 1$. The transfer functions of $\delta_c, t_c$ are given by
\begin{equation*}
\sigma_\delta(x)= \begin{cases} 1+\eps & \textrm{if } x > 2M+S \\
1 & \textrm{if }  x \in [M, 2M+S] \\
0 & \textrm{if }  x \in [-M,M] \\
-\mu & \textrm{if } x \in [-2M-S,-M] \\
-\mu-\eps & \textrm{if } x < -2M-S
\end{cases},
\quad \quad
\sigma_t(x) = \begin{cases} 1 & \textrm{if } x > 1 \\ 0 & \textrm{if } -\mu \le x \le 1 \\ -1 & \textrm{if } x < -\mu 
\end{cases}
\end{equation*}
and shown in figure \ref{sigmas}.

\begin{figure}[h]
\begin{center}
\includegraphics[width=2.5in]{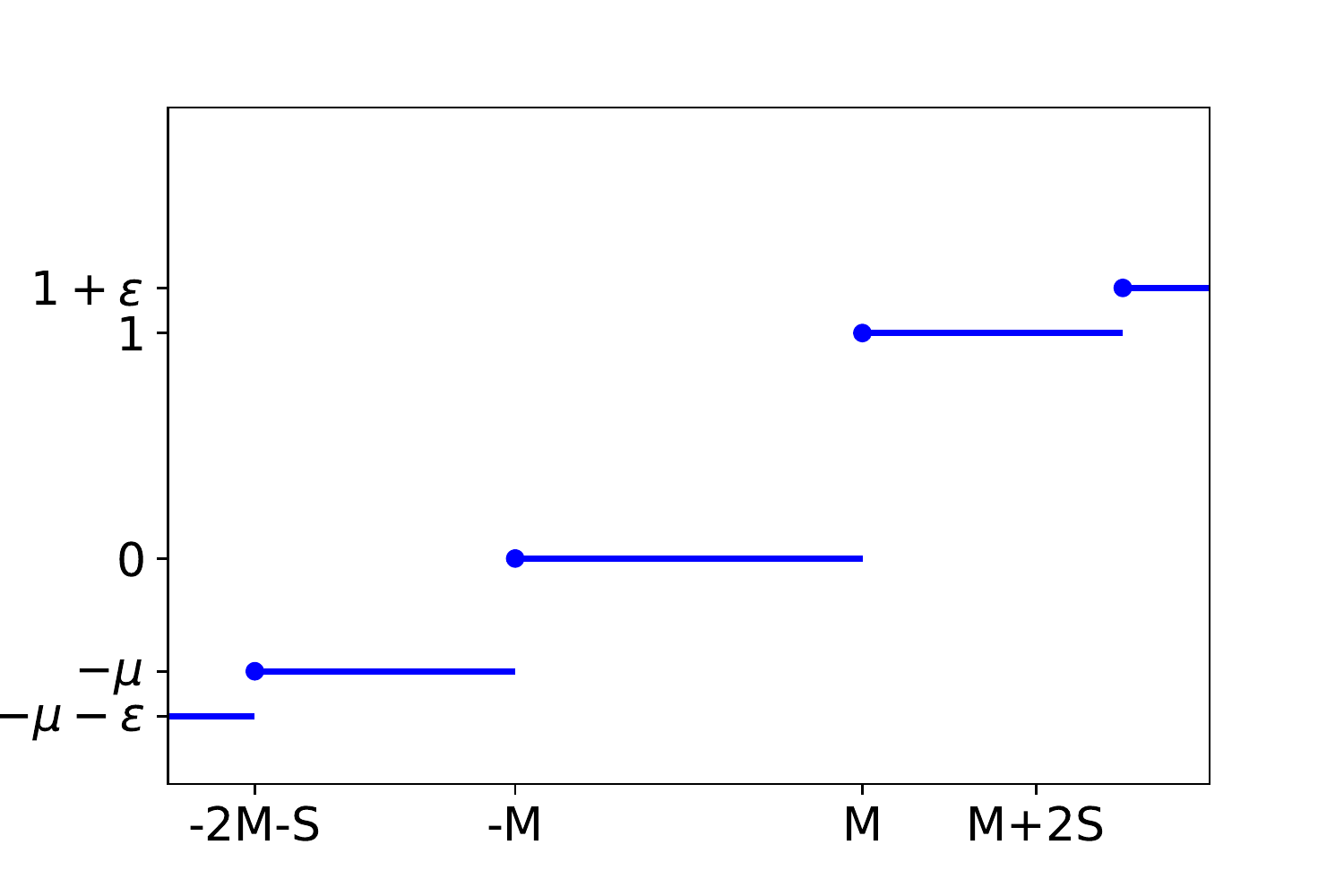}%
\includegraphics[width=2.5in]{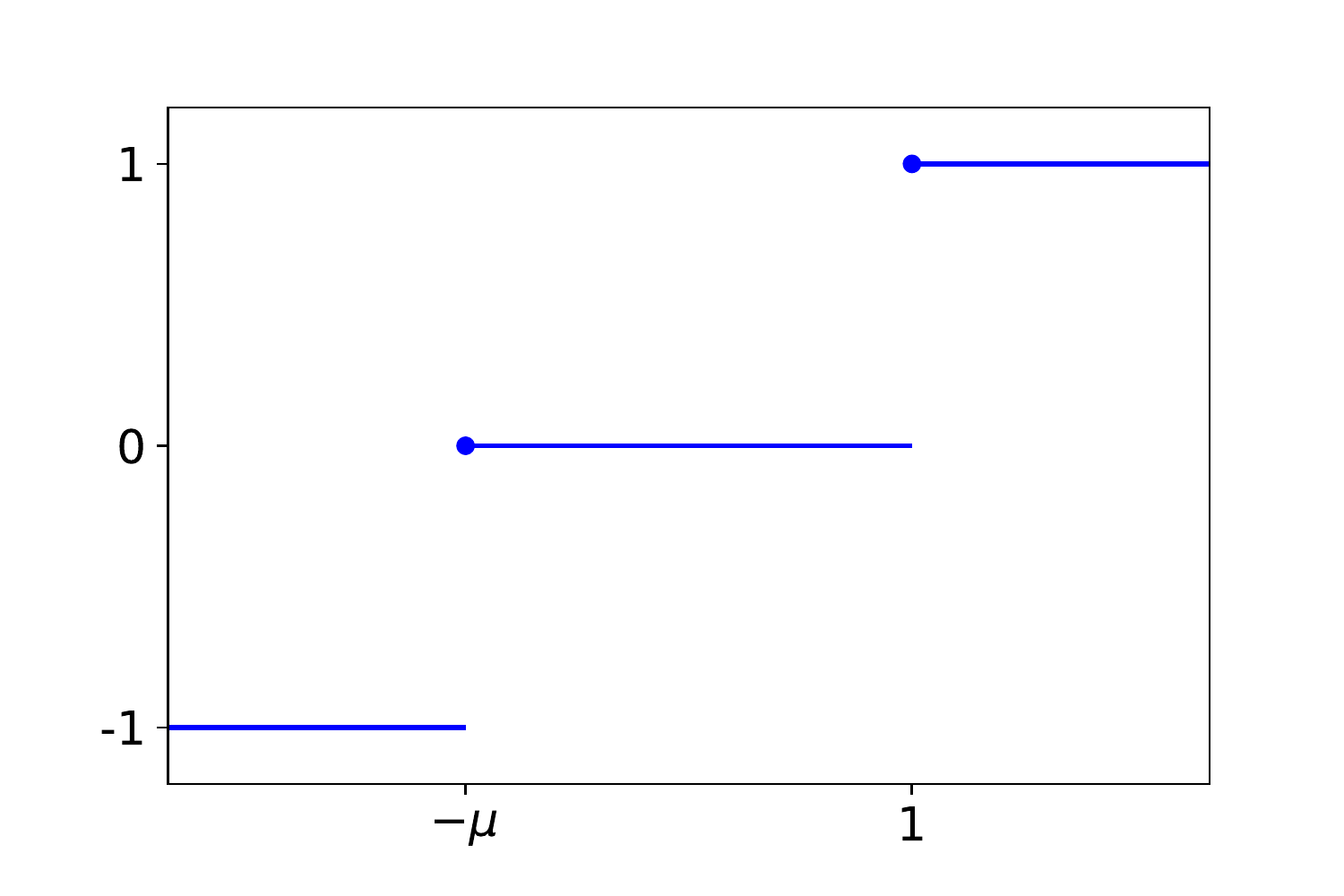}
\end{center}
\caption{Left: Transfer function  $\sigma_v$. Right: Transfer function  $\sigma_t$. \label{sigmas}}
\end{figure}

If the supervisory signal is $s_c=1$, $t_c$ is initialized at $0$ and $h_c \le S$ then $M \le H_c \le 2M+S$, so $\delta_c=1$ and $t_c=0$.
This yeilds the update cycle 
$$\delta_c=1 \rightarrow t_c=0 \rightarrow \delta_c=1, $$ 
and $\delta_c$ is constant at value $1$.

If $h_c>S$ then $H_c>2M+S$ so that $\delta_c=1+\eps$ and $t_c=1$. Then, $H_c=h_c \in [-M,M]$ so that $\delta_c=0$. This yields the update cycle:
$$\delta_c=1+\eps \rightarrow t_c=1 \rightarrow \delta_c=0 \rightarrow t_c=0 \rightarrow \delta_c=1+\eps,$$
so that $\delta_c$ oscillates between $0$ and $1+\eps$.

Conversely if the supervisory signal is $s_c=-1$ and $h_c \ge -S$ then $-2M-S \le H_c \le -M$ so $\delta_c=-\mu$
and $t_c=0$ yielding $$\delta_c=-\mu \rightarrow t_c=0 \rightarrow \delta_c=-\mu,$$ 
and $\delta_c$ is constant at $-\mu$. 
If $h_c < -S$ then $H_c < -2M-S$, $\delta_c=-\mu-\eps$ causing $t_c=-1$ and  
$H_c=h_c$ so that $\delta_c=0$ and we get the update cycle:
$$\delta_c=-\mu-\eps \rightarrow t_c=-1 \rightarrow v_c=0 \rightarrow t_c=0 \rightarrow \delta_c=-\mu-\eps,$$
so that $\delta_c$ oscillates between $0$ and $-\mu-\eps$.

In summary the activity of $\delta_c$ is input dependent. If $sign(\delta_c) h_c  <S$ then $\delta_c = 1, -\mu$ depending on the presented class,
and $\Delta W_{ic} = x_i \delta_c$.
If $sign(o)c) h_c  > S$ then $\delta_c$ oscillates and periodically visits the state $\delta_c=0$ and no update of synapses
connecting to $\delta_c$ occurs. Ignoring the oscillation this circuit can be thought of as implementing the rule in \eqref{delta_outa}.
The activity of $\delta_c$ is precisely the signal that needs to be propagated backwards to the input layer.
This network is illustrated in figure \ref{circuit1}.

\begin{figure}[h]
\begin{center}
\includegraphics[width=3in]{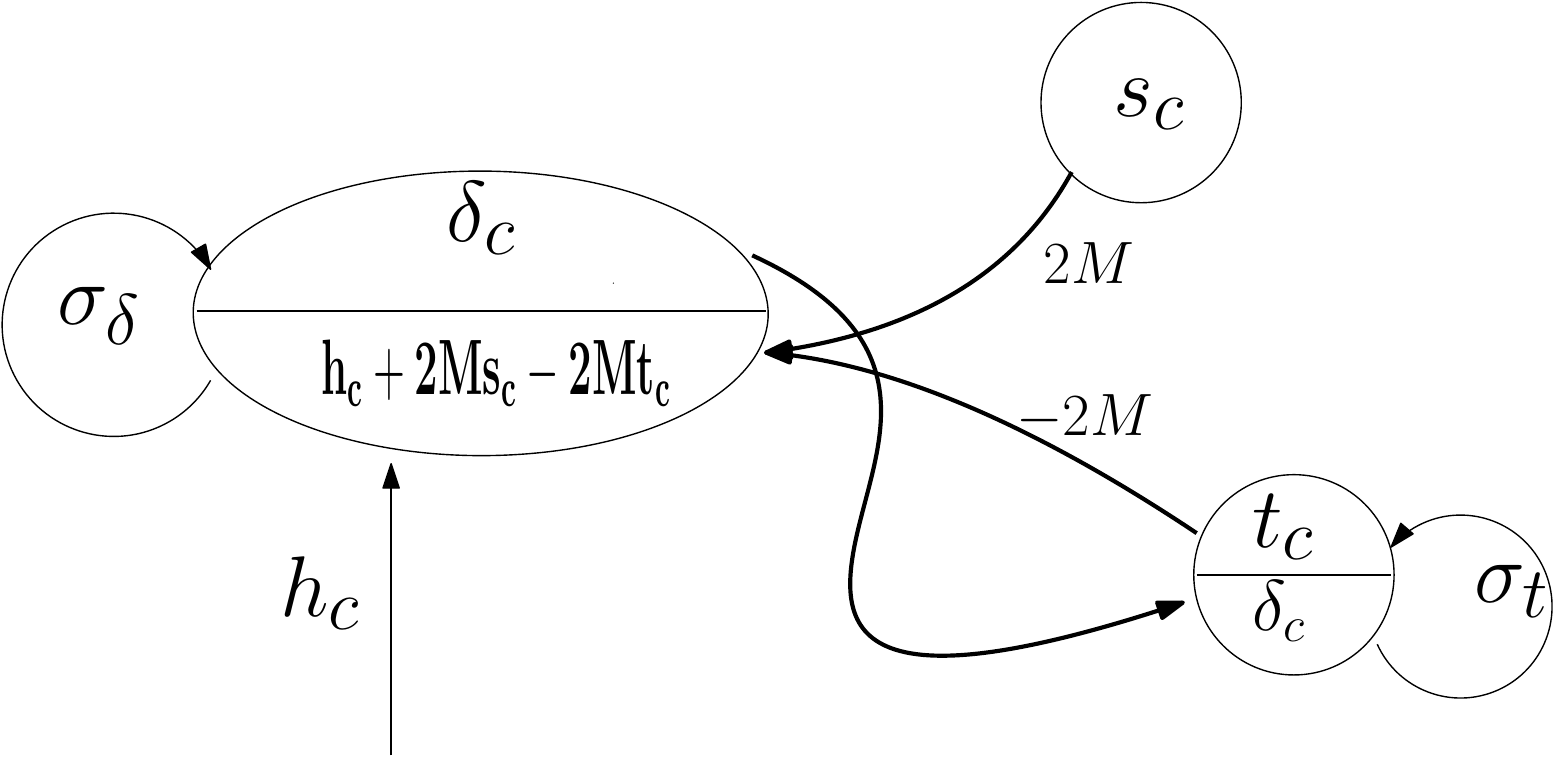}
\end{center}
\caption{The circuit including the supervisory neuron $s_c$, the top layer neuron
$\delta_c$ and the control neuron $t_c$. \label{circuit1}}
\end{figure}

\subsection*{Computing the shutdown of the feedback computation}
We describe a local network to compute the expression in equation \ref{delthresh}.
We add two control neurons $u_{l,i}, v_{l,i}$ with a simple threshold
activation and input from $x_{l,i}$. We set $u_{l,i}=\bf{1}[{x_{l,i} \ge 1}]$
and $v_{l,i} = \bf{-1}[{x_{l,i} \le -1}]$. Only one of these units can be active.
The input to unit $l,i$ is the top-down feedback 
$\delta_{l,i}$, and we add $- 2K u_{l,i} + 2K v_{l,i}$.
Let 
$$\sigma_\delta(u) = \begin{cases} u & \text{if } u \ge -K \\ 0 & \text{otherwise} \end{cases}
$$
 and set
$$ \tilde \delta_{l,i} = \sigma_\delta( \delta_{l,i} -2 K u_{l,i} + 2K v_{l,i}).$$
Assuming $\delta_{l,i}$ is bounded between $-K,K$ then if either of the units $u_{l,i}, v_{l,i}$ 
are active $ \delta_{l,i} -2 K u_{l,i} + 2K v_{l,i} \le -K$ and $\tilde \delta_{l,i}=0$,
otherwise $\tilde \delta_{l,i} = \delta_{l,i}.$
In terms of timing, before the feedback signal arrives at the unit, 
but after the value of $x_{l,i}$ is used to update the weights $W_l$,
the original
feedforward activity $x_{l,i}$ activates the units $u_{i,l}$ and $v_{i,l}$ so
that their input is added to the incoming feedback signal $\delta_{i,k}$.
This sequence of updates is illustrated in figure \ref{bp_circ}.
\begin{figure}[h]
\center
\includegraphics[height=3in]{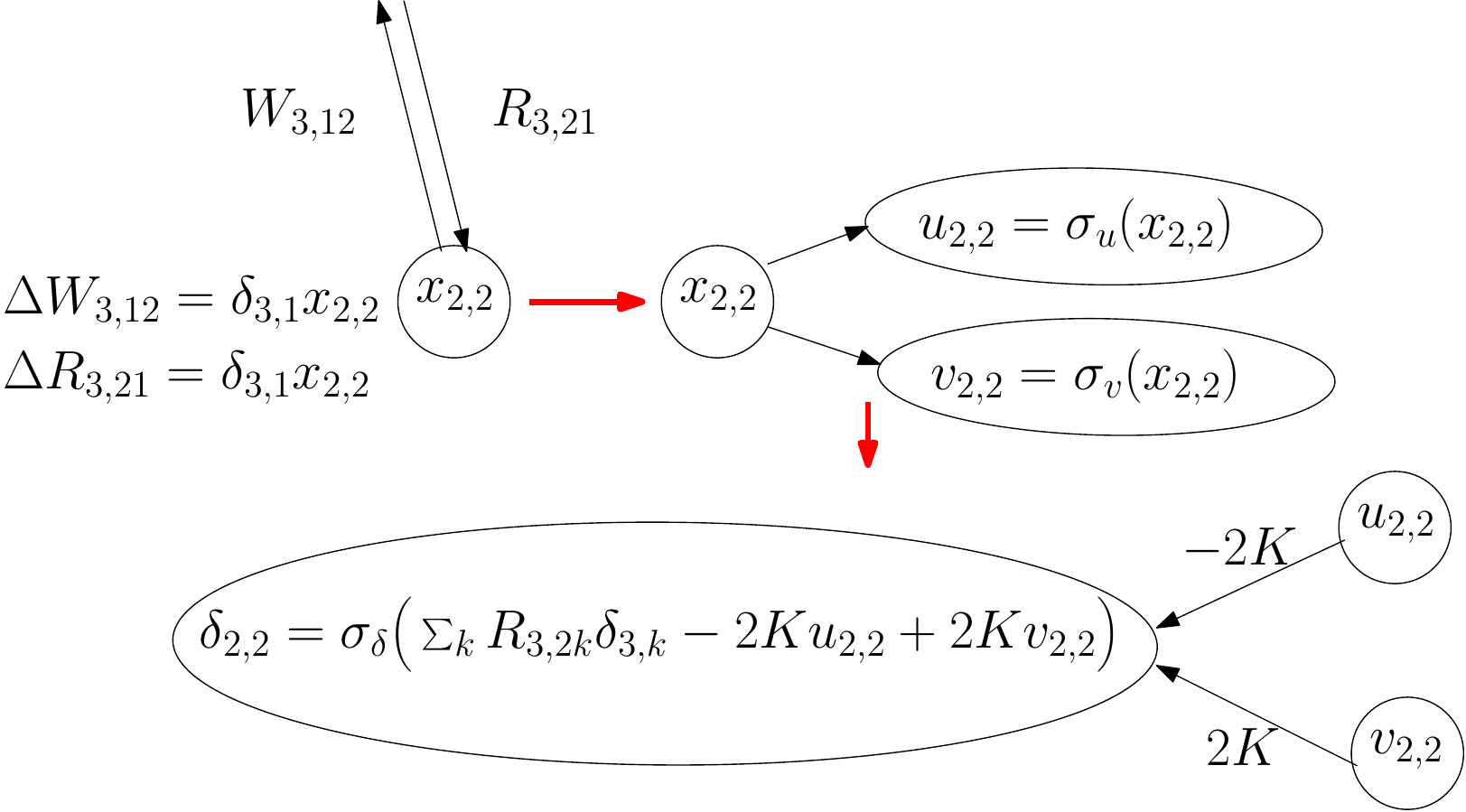}
\caption{First weights connecting unit $x_{2,2}$ to layer above get updated based on the
value of $x_{2,2}$. Second the activation of $u_{2,2}$ and $v_{2,2}$ is triggered. Third
the computation of $\delta_{2,2}$ is performed in terms of $\delta_{3,k} and R_{3,2k}$
and set to 0 if either $u_{2,2}$ or $v_{2,2}$ are active.}\label{bp_circ}
\end{figure}

\section*{Appendix 2: Statement of theorems and proofs}

\subsection*{Two layer network}

Let $T = U \Lambda_T V^t$ be the SVD of $T$. Assuming $n_2 \le n_0$ we have
$U \in \R^{n_2 \times n_2},  V \in \R^{n_0 \times n_2}$ and $\Lambda_T \in \R^{n_2 \times n_2}$ positive diagonal. We set the initial condition for $\rb$ by first picking $n_2$ orthogonal vectors
in $\R^{n_1}$ yielding an $n_1\times n_2$  matrix $S$ with orthogonal columns, and writing $\rb(0)=S \Lambda_R U^t$
with $\Lambda_R$ diagonal and positive.
This is a restricted initial condition since for the most general initial condition, if we are free to choose $S$, $\Lambda_R$ would be upper triangular.
Since $\wa(0)=\wb(0)=0$ and $\rb(0)T=S \Lambda_R \Lambda_T V^t$ then we will always
have $\wa=S\La V^t$
for some $\La \in \R^{n_2 \times n_2}$ diagonal,
and $\wb = U\Lb S^t$ for some $\Lb \in \R^{n_2 \times n_2}$ diagonal.
 This setup yields:
$$ \wb\wa = U \Lb \La V^t, \quad (T-W_2W_1) =  U (\Lambda - \Lb\La) V^t. $$
So that the matrix equations in \eqref{mbpa} decouple into $n_2$ pairs of scalar ODE's
one for each of the $n_2$ columns of $V$.

\begin{align*}
\dot \Lambda_2 = & (\Lt - \Lb\La) \La \nonumber \\
\dot \Lambda_1 = & (\Lr + \eps \Lb) (\Lt - \Lb\La).
\end{align*}

We analyze one such pair of scalar equations: 
\begin{align}\label{dmbpas}
\dot \lambda_2 = & (\lambda_T - \lambda_2\lambda_1) \lambda_1 \nonumber \\
\dot \lambda_1 = & (\lambda_R+ \eps \lambda_2) (\lambda_T - \lambda_2\lambda_1),
\end{align}
assuming without loss of generality that $0 < \lambda_T \le 1$ and $\lambda_R << \lambda_T.$

\begin{thm}
Define $e=\lt-\lb\la$, and 
assume $\la(0)=0, \lb(0)=0$. Then  $e^2$ converges exponentially fast to 0 and the rate of convergence
increases as $\epsilon$ increases.
\end{thm}

\begin{proof}[Proof of Theorem 1]
For $k=2$ equation \eqref{e2} reduces to
\begin{equation}\label{rate}
\dot{e^2} = -2e^2 (\la^2 + \frac{1}{2} \la^2 + \frac{\epsilon}{2}\lb^2).
\end{equation}
Since $\lt,\lr>0$ and $\la(0)=\lb(0)=0$ we see that $\la$ and $\lb$ are increasing in time,
and so after a finite time are uniformly bounded away from $0$. This implies that the error
goes to $0$ exponentially fast. We also note that  $\la \lb = \lt$ is a stationary point, so that $\la \lb \le \lt$
always holds. If we can show that the factor of $-e^2$ increases with $\epsilon$ then the rate of convergence of $e$ to zero
increases with $\epsilon.$

Solving for $\lambda_{i}$ in equation \eqref{lll} and taking the positive solution we can write
\begin{align}\label{ww}
\lambda_i = & G(\lambda_{i-1},\lambda_{R,i},\eps) \\ 
\lambda_{i-1} = & H(\lambda_i,\lambda_{R,i},\eps), \nonumber
\end{align}
where
\begin{align}\label{FG}
F(x,\lambda_R,\eps) =  &\sqrt{\lambda_R^2+\eps x^2}, \nonumber \\
G(x,\lambda_R,\eps) = &\frac{F(x,\lambda_R,\eps)-\lambda_R}{\epsilon} = \frac{x^2}{F(x,\lambda_R,\eps)+\lambda_R},  \text{ and} \nonumber \\
H(x,\lambda_R,\eps) = & \sqrt{2 \lambda_R x + \eps x^2}.
  \end{align}
  
If we show that $\lambda_1$ increases with
$\epsilon$, then, by \eqref{ww},  $\sqrt{\eps} \lb = \sqrt{\lr^2/\eps + \la^2} - \lr/\sqrt{\eps},$
 is easily seen to increase with $\epsilon$ since $\la$ and $\lambda_R$ are positive.
(In general, for any positive function $g(\eps)$ that is increasing in $\epsilon$ the function
$f(\eps) = \sqrt{c^2/\eps + g(\eps)} - c/\sqrt{\eps}$ is increasing w.r.t to $\epsilon$.) 
Thus $\eps \lambda_2^2$ is increasing in $\epsilon$ and the factor of $-e^2$ in equation \eqref{rate} is increasing in $\epsilon.$

Since $\lb = G(\la,\lr,\eps)$
we get the scalar equation for $\la$ as
\begin{align}\label{bb}
\dot \lambda_1 = & F(\la,\lr,\eps)(\lt-G(\la,\lr,\eps)\la) \equiv f_1(\la,\eps)
\end{align}
Let  $h>0$ be a small increment and write,
$$ \dot \lambda_\eps = f(\lambda_\eps,\eps), \quad \dot \lambda_{\eps+h} = f(\lambda_{\eps+h}, \eps+h).$$
Denote $\Delta=\lambda_{\eps+h} - \lambda_\eps$ then to first order
$$ \dot \Delta = \frac{\partial f}{\partial \lambda} (\lambda_\eps,\eps) \Delta + \frac{\partial f}{\partial \eps} (\lambda_\eps,\eps).$$
So if the second term is positive and the initial condition $\Delta(0)=0$ then, using the formula for the solution to first order ODE's, $\Delta$ is always positive, so that $\lambda_\eps$ is increasing in $\eps$. 
In our setting by \eqref{FG}
\begin{align*}
\frac{\partial F(x,\lambda_R,\eps)}{\partial \eps} = & \frac{x^2}{2F(x,\lambda_R,\eps)} > 0, \\
\frac{\partial G(x,\lambda_R,\eps)}{\partial \eps} = & -\frac{x^4}{2(F(x,\lambda_R,\eps)+\lambda_R)^2F(x,\lambda_R,\eps)} <0, 
\end{align*}
Consequently:
\begin{equation*}
\frac{\partial  f_1}{\partial \eps} =  \frac{e\la^2}{2F(\la,\lambda_R,\eps)} + F\la \frac{\la^4}{2(F(\la,\lambda_R,\eps)+\lambda_R)^2F(\la,\lambda_R,\eps)} > 0.
\end{equation*}
Thus $\lambda_1$ is increasing in $\epsilon.$
\end{proof}

\subsection*{Deeper networks}
Let $T = U \Lambda_T V^t$ be the SVD of $T$. Assuming $n_k \le n_0,n_1\ldots,n_{k-1}$ we have
$U \in \R^{n_k \times n_k},  V \in \R^{n_0 \times n_k}$ and $\Lambda_T \in \R^{n_k \times n_k}$ positive diagonal.
We set the initial condition for $R_k$ by first picking $n_{k}$ orthogonal vectors
in $\R^{n_{k-1}}$  yielding an $n_{k-1}\times n_{k}$ matrix $U_{k}$ and then writing $R_k(0)=U_k \Lambda_{R,k} U^t$
with $\Lambda_{R,k} \in \R^{n_k \times n_k}$ diagonal and positive and $U_k \in \R^{n_{k-1} \times n_{k}}$ with orthogonal columns.
For general $i$ choose $U_{i} \in \R^{n_{i-1} \times n_k}$ orthogonal, $\Lambda_{R,i} \in \R^{n_{k} \times n_{k}}$ positive and diagonal,
and set $R_{i}(0) = U_{i} \Lambda_{R,i} U_{i+1}^t$. Write $W_i = U_{i+1} \Lambda_{i} U_i^t, i=1\ldots,k$ and with $U_{k+1}=U$,
and assume $\Lambda_i(0)=0$ then the system \eqref{mk} decouples into $n_k$ scalar equations
one for each of the directions in $U,V$:
\begin{align}\label{deep}
\dot \lambda_k & = e \lambda_{k-1} \cdot \lambda_1 \nonumber \\
& \vdots \nonumber \\
\dot \lambda_i &=  (\lambda_{R,k}+\eps \lambda_k) \cdots (\lambda_{R,i+1} + \eps \lambda_{i+1}) e \lambda_{i-1} \cdots \lambda_1 \nonumber \\
& \vdots \nonumber \\
\dot \lambda_1 &= (\lambda_{R,k}+\eps \lambda_k) \cdots (\lambda_{R,2} + \eps \lambda_{2}) e,
\end{align}
with $\lambda_i(0)=0$ and $\lambda_{R,i}>0$ random. This implies that $\lambda_i$ are increasing and
$e$ is always positive.
Multiplying the $i$'th equation by $(\lambda_{R,i} + \eps \lambda_i)$ and the $i-1$'th equation by $\lambda_{i-1}$ we get the equality
$$ \dot \lambda_i (\lambda_{R,i} + \eps \lambda_i) = \dot \lambda_{i-1} \lambda_{i-1}, i=2,\ldots,k.$$
Since $\lambda_i(0)=0$ we can integrate and get
\begin{equation}\label{lll}
\lri \lambda_i + \frac{\eps}{2} \lambda_i^2 = \frac{1}{2} \lambda_{i-1}^2
\end{equation}
Rewriting $e=\lambda_T-\prod_{i=1}^k \lambda_i,$ we have
\begin{align}\label{e2}
\dot {\frac{e^2}{2}} =  - e \sum_{i=1}^k \dot \lambda_i \prod_{j \ne i} \lambda_j 
 =& -e^2 \sum_{i=1}^k \prod_{j=1}^{i-1} \lambda_j^2 \prod_{j=i+1}^k \lambda_j (\lambda_{R,j} + \eps \lambda_j)  \nonumber \\
=&  -e^2 \sum_{i=1}^k \prod_{j=1}^{i-1} \lambda_j^2 \prod_{j=i+1}^k (\frac{\eps}{2} \lambda_j^2 + \frac{1}{2} \lambda_{j-1}^2)
\end{align}
where the third equality follows from \eqref{lll}.
It follows that for each $\epsilon$ the error converges exponentially fast to $0$ and if
it can be shown that $\lambda_i$ are increasing in $\epsilon$ the convergence rate is increasing with $\epsilon.$
We are able to show this  for $k=3$.

\begin{thm} For $k=3$, assume $\lambda_1(0)=0, \lambda_2(0)=0$. Assume $\lambda_{R,2},\lambda_{R,3}< \delta << \lambda_T<1$,
and assume $\lambda_{R,2} > \frac{1+\sqrt{1+\eps}}{2}\lambda_{R,3}$, then $\lambda_1, \lambda_2$ and $\lambda_3$
are increasing in $\epsilon$ so that $e$ converges faster to 0 as $\epsilon$ increases.
\end{thm}

\begin{proof}[Proof of Theorem 2]
With $k=3$ equation \eqref{e2} reduces to 
\begin{equation}\label{e23}
\dot {\frac{e^2}{2}} = - e^2 \left[  \left(\frac{\epsilon}{2} \lb^2 + \frac{3}{2} \la^2 \right)
\cdot  \left(\frac{\epsilon}{2} \lambda_3^2 + \frac{1}{2} \lb^2 \right) + \la^2\lb^2 \right].
\end{equation}
We need to show that $\lambda_1$ and $\lambda_2$ are increasing in $\epsilon$.
The term $\epsilon \lambda_3^2$ can be handled just like the term $\epsilon \lambda_2^2$ in the case $k=2$.

Write $H(\lambda_2) = H(\lambda_2,\lambda_{R,2},\eps)$, $F(\lambda_2)=F(\lambda_2,\lambda_{R,3},\eps), G(\lambda_2)=G(\lambda_2,\lambda_{R,3},\eps)$  functions of $\lambda_2$.
Furthermore set 
\begin{align*}
G_2(\lambda_1)= &G(\lambda_1,\lambda_{R,2},\eps), \\
G_3(\lambda_1)= &G(G_2(\la),\lambda_{R,3},\epsilon), \\
F_2(\lambda_1)= &F(\lambda_1,\lambda_{R,2},\eps), \\
F_3(\lambda_1)=&F(G_2(\lambda_1),\lambda_{R,3},\epsilon)
\end{align*}
 all functions of $\lambda_1$.
Write the equations for $\lambda_1,\lambda_2$ with all other variables eliminated:
\begin{align}
\dot \lambda_2 =& F(\lambda_2) \left(\lambda_T - G(\lambda_2) \lambda_2 H(\lambda_2)\right) H(\lambda_2) \nonumber \\
\dot \lambda_1 =& F_3(\lambda_1)F_2(\lambda_1) \left(\lambda_T - G_3(\lambda_1) G_2(\lambda_1)\lambda_1\right)  \label{ll3}
\end{align}
We have
\begin{align*}\label{partials}
\frac{d F_2}{d \eps} = & \frac{\lambda_1^2}{2 F_2} \\
\frac{d G_2}{d \eps} = & \frac{-\lambda_1^4}{2(F_2+\lambda_{R,2})^2 F_2} \\
\frac{d F_3}{d \eps} = &  \frac{\partial F_3}{\partial G_2} \frac{\partial G_2}{\partial \eps} + \frac{\partial F_3}{\partial \eps} 
= \frac{G_2^2}{2F_3} - \frac{\eps G_2}{F_3}\frac{d G_2}{d\eps} \\
\frac{d G_3}{d \eps} = & - \frac{G_2^4}{2(F_3 + \lambda_{R,3})^2 F_3} - \frac{\eps G_2}{F_3} \frac{\lambda_1^4}{2(F_2+\lambda_{R,2})^2 F_2} \\
\frac{d H}{d\eps} = &\frac{\lambda_2^2}{2 H}
\end{align*}

For $\lambda_1$ 
 the derivative of the right hand side of \eqref{ll3} with respect to $\epsilon$ is
$$\frac{d F_3}{d \eps} F_2 e + F_3 \frac{d F_2}{d \eps} e - F_3 F_2  \left[\frac{d G_3}{d \eps} G_2 \lambda_1 + G_3 \frac{d G_2}{d \eps} \lambda_1 \right].$$
Since $F_2,G_2,F_3,G_3 > 0$ and the derivatives of $G_2,G_3$ with respect to $\epsilon$ are negative the second and third terms are
positive. It is left to show that $\frac{d F_3}{d \eps}>0$.
Substituting  $G_2 = \frac{\lambda_1^4}{(F_2+\lambda_{R,2})^2}$ in the first term we have:
$$
\frac{G_2^2}{2F_3} - \frac{\eps G_2}{F_3}\frac{d G_2}{d\eps} =
\frac{\lambda_1^4}{2F_3(F_2+\lambda_{R,2})^2} \left[ 1 - \frac{\eps G_2}{F_2} \right]
=\frac{\lambda_1^4}{2F_3(F_2+\lambda_{R,2})^2} \frac{\lambda_{R,2}}{F_2}>0.$$

For $\lambda_2$  the derivative of the right hand side of \eqref{ll3} with respect to $\epsilon$:
$$ \frac{\lambda_2^2}{2 F} e H + F \left[ \frac{\lambda_2^4}{2(F+\lambda_{R,3})^2 F} \lambda_2 H - G \frac{\lambda_2^2}{2H} \lambda_2 \right] H + F e \frac{\lambda_2^2}{2H} = T_1+T_2+T_3.$$
The first and third terms are positive.
For the second term we have substituting $G = \frac{\lambda^2_2}{F+\lambda_{R,3}}$,
\begin{align*}
T_2= \left[ \frac{\lambda_2^4}{2(F+\lambda_{R,3})^2 F} \lambda_2 H - G \frac{\lambda_2^2}{2H} \lambda_2 \right]
= &\frac{\lambda_2^5}{2} \left [ \frac{H}{(F+\lambda_{R,3})^2 F} - \frac{1}{H(F+\lambda_{R,3})} \right]\\
= &\frac{\lambda_2^5}{(F+\lambda_{R,3})^2 F H} [H^2 - (F+\lambda_{R,3}) F].
\end{align*}
The last factor is 
\begin{equation*}
2\lambda_{R,2} \lambda_2  - \lambda_{R,3}^2-\lambda_{R,3}\sqrt{\lambda_{R,3}^2 + \eps\lambda_2^2}.
\end{equation*}
Assume for small $\delta < 1$ that
 $\lambda_{R,2},\lambda_{R,3}< \delta << \lambda_T$ then an order computation shows that as long as $\lambda_2 \le \delta$ then
$T_1+T_3 = o(\delta^2)$ whereas $0>T_2 = o(\delta^5)$ and so the sum is positive.
When $\lambda_2> \lambda_{R,2}, \lambda_{R,3}$,  then 
the expression in \eqref{last} is bounded below by
$2\lambda_{R,2} \lambda_2 - \lambda_{R,3} \lambda_2 - \lambda_{R,3} \sqrt{1+\eps} \lambda_2>0$
if $\lambda_{R,2} > \frac{1+\sqrt{1+\eps}}{2}\lambda_{R,3}$.
Thus $T_2>0$.
In summary both $\lambda_1,\lambda_2$ are increasing in $\epsilon$ as long as 
$\lambda_{R,2},\lambda_{R,3}< \delta << \lambda_T$.
\end{proof}

\bibliography{../../BIB}

\bibliographystyle{plain}
\end{document}